%% file: main.tex
\documentclass[letterpaper,11pt]{article}
\usepackage[margin=1in]{geometry}  

\usepackage{listings}
\usepackage{amsmath}
\usepackage{amsthm}
\usepackage{tikz}
\usepackage{caption,subcaption}
\usepackage{array}
\usepackage{mdwmath}
\usepackage{multirow}
\usepackage{mdwtab}
\usepackage{eqparbox}
\usepackage{multicol}
\usepackage{amsfonts}
\usepackage{tikz}
\usepackage{multirow,bigstrut,threeparttable}
\usepackage{amsthm}
\usepackage{array}
\usepackage{bbm}
\usepackage{epstopdf}
\usepackage{mdwmath}
\usepackage{mdwtab}
\usepackage{eqparbox}
\usetikzlibrary{topaths,calc}
\usepackage{latexsym}
\usepackage{cite}
\usepackage{amssymb}
\usepackage{bm}
\usepackage{amssymb}
\usepackage{graphicx}
\usepackage{mathrsfs}
\usepackage{epsfig}
\usepackage{psfrag}
\usepackage{setspace}
\usepackage[
            CJKbookmarks=true,
            bookmarksnumbered=true,
            bookmarksopen=true,
            colorlinks=true,
            citecolor=red,
            linkcolor=blue,
            anchorcolor=red,
            urlcolor=blue
            ]{hyperref}
\usepackage[ruled]{algorithm2e}
\usepackage{algpseudocode}
\usepackage{stfloats}

\usepackage{mathtools}
\usepackage{blkarray}
\usepackage{multirow,bigdelim,dcolumn,booktabs}

\usepackage{xparse}
\usepackage{tikz}
\usetikzlibrary{calc}
\usetikzlibrary{decorations.pathreplacing,matrix,positioning}

\usepackage[T1]{fontenc}
\usepackage[utf8]{inputenc}
\usepackage{mathtools}
\usepackage{blkarray, bigstrut}
\usepackage{gauss}

\newcommand*{\BraceAmplitude}{0.4em}%
\newcommand*{\VerticalOffset}{0.5ex}%
\newcommand*{\HorizontalOffset}{0.0em}%
\newcommand*{\blocktextwid}{3.0cm}%
\NewDocumentCommand{\InsertLeftBrace}{%
	O{} 
	O{\HorizontalOffset,\VerticalOffset} 
	O{\blocktextwid} 
	m   
	m   
	m   
}{%
	\begin{tikzpicture}[overlay,remember picture]
	\coordinate (Brace Top)    at ($(#4.north) + (#2)$);
	\coordinate (Brace Bottom) at ($(#5.south) + (#2)$);
	\draw [decoration={brace, amplitude=\BraceAmplitude}, decorate, thick, draw=black, #1]
	(Brace Bottom) -- (Brace Top) 
	node [pos=0.5, anchor=east, align=left, text width=#3, color=black, xshift=\BraceAmplitude] {#6};
	\end{tikzpicture}%
}%
\NewDocumentCommand{\InsertRightBrace}{%
	O{} 
	O{\HorizontalOffset,\VerticalOffset} 
	O{\blocktextwid} 
	m   
	m   
	m   
}{%
	\begin{tikzpicture}[overlay,remember picture]
	\coordinate (Brace Top)    at ($(#4.north) + (#2)$);
	\coordinate (Brace Bottom) at ($(#5.south) + (#2)$);
	\draw [decoration={brace, amplitude=\BraceAmplitude}, decorate, thick, draw=black, #1]
	(Brace Top) -- (Brace Bottom) 
	node [pos=0.5, anchor=west, align=left, text width=#3, color=black, xshift=\BraceAmplitude] {#6};
	\end{tikzpicture}%
}%
\NewDocumentCommand{\InsertTopBrace}{%
	O{} 
	O{\HorizontalOffset,\VerticalOffset} 
	O{\blocktextwid} 
	m   
	m   
	m   
}{%
	\begin{tikzpicture}[overlay,remember picture]
	\coordinate (Brace Top)    at ($(#4.west) + (#2)$);
	\coordinate (Brace Bottom) at ($(#5.east) + (#2)$);
	\draw [decoration={brace, amplitude=\BraceAmplitude}, decorate, thick, draw=black, #1]
	(Brace Top) -- (Brace Bottom) 
	node [pos=0.5, anchor=south, align=left, text width=#3, color=black, xshift=\BraceAmplitude] {#6};
	\end{tikzpicture}%
}%

\usetikzlibrary{patterns}

\definecolor{cof}{RGB}{219,144,71}
\definecolor{pur}{RGB}{186,146,162}
\definecolor{greeo}{RGB}{91,173,69}
\definecolor{greet}{RGB}{52,111,72}

\theoremstyle{plain}
\newtheorem{theorem}{Theorem}

\newtheorem{lemma}{Lemma}
\newtheorem{remark}{Remark}
\newtheorem{corollary}{Corollary}
\newtheorem{definition}{Definition}

\def \bP {\mathbb{P}}
\def \bE {\mathbb{E}}
\def \bR {\mathbb{R}}

\def \var {\mathsf{Var}}
\def\1{\mathbbm{1}}

\usepackage{xspace}

\newcommand{\stepa}[1]{\overset{\rm (a)}{#1}}
\newcommand{\stepb}[1]{\overset{\rm (b)}{#1}}
\newcommand{\stepc}[1]{\overset{\rm (c)}{#1}}
\newcommand{\stepd}[1]{\overset{\rm (d)}{#1}}
\newcommand{\stepe}[1]{\overset{\rm (e)}{#1}}

\definecolor{myblue}{rgb}{.8, .8, 1}
\definecolor{mathblue}{rgb}{0.2472, 0.24, 0.6} 
\definecolor{mathred}{rgb}{0.6, 0.24, 0.442893}
\definecolor{mathyellow}{rgb}{0.6, 0.547014, 0.24}

\newcommand{\calA}{{\mathcal{A}}}
\newcommand{\calB}{{\mathcal{B}}}
\newcommand{\calC}{{\mathcal{C}}}

\newcommand{\calE}{{\mathcal{E}}}
\newcommand{\calF}{{\mathcal{F}}}

\newcommand{\calN}{{\mathcal{N}}}

\newcommand{\calP}{{\mathcal{P}}}

\newcommand{\calX}{{\mathcal{X}}}

\usepackage{cleveref}
\crefname{lemma}{Lemma}{Lemmas}
\Crefname{lemma}{Lemma}{Lemmas}
\crefname{thm}{Theorem}{Theorems}
\Crefname{thm}{Theorem}{Theorems}
\crefformat{equation}{(#2#1#3)}

\begin{document}

\title{Optimal No-regret Learning in Repeated First-price Auctions}
\author{Yanjun Han, Zhengyuan Zhou, Tsachy Weissman\thanks{Y. Han is with the Courant Institute of Mathematical Sciences and the Center for Data Science, New York University, email: \url{yanjunhan@nyu.edu}. Z. Zhou is with the Stern School of Business, New York University, email: \url{zzhou@stern.nyu.edu}. T. Weissman is with the Department of Electrical Engineering, Stanford University, email: \url{tsachy@stanford.edu}. This project was supported in part by NSF awards CCF-2106467 and CCF-2106508. Y. Han and T. Weissman were partially supported by the Yahoo Faculty Research and Engagement Program.}}
\maketitle
\begin{abstract}
We study online learning in repeated first-price auctions where a bidder, only observing the winning bid at the end of each auction, learns to adaptively bid in order to maximize her cumulative payoff. To achieve this goal, the bidder faces censored feedback: if she wins the bid, then she is not able to observe the highest bid of the other bidders, which we assume is \textit{iid} drawn from an unknown distribution. In this paper, we develop the first learning algorithm that achieves a near-optimal $\widetilde{O}(\sqrt{T})$ regret bound, by exploiting two structural properties of first-price auctions, i.e. the specific feedback structure and payoff function. 

We first formulate the feedback structure in first-price auctions as partially ordered contextual bandits, a combination of the graph feedback across actions (bids), the cross learning across contexts (private values), and a partial order over the contexts. We establish both strengths and weaknesses of this framework, by showing a curious separation that a regret nearly independent of the action/context sizes is possible under stochastic contexts, but is impossible under adversarial contexts. In particular, this framework leads to an $O(\sqrt{T}\log^{2.5}T)$ regret for first-price auctions when the bidder's private values are \emph{iid}.  

Despite the limitation of the above framework, we further exploit the special payoff function of first-price auctions to develop a sample-efficient algorithm even in the presence of adversarially generated private values. We establish an $O(\sqrt{T}\log^3 T)$ regret bound for this algorithm, hence providing a complete characterization of optimal learning guarantees for first-price auctions. 
	
\end{abstract}
\tableofcontents

\input{intro.tex}
\input{rest.tex}
\input{adversarial.tex}
\input{applications.tex}

\section*{Acknowledgements}
We would like to thank Ruiyuan Huang, Yuxiao Wen, and Tiancheng Yu for spotting technical issues in the previous versions. We are also grateful to the area editor, the associate editor, and three anonymous reviewers for various suggestions on improving this paper. Finally, we would like to thank Harikesh S. Nair, Erik Ordentlich, and Caio Waisman for many helpful discussions on the problem setup. 

\appendix
\input{appendix.tex}

\bibliographystyle{alpha}
\bibliography{di.bib}

\end{document}

%% file: intro.tex
\section{Introduction}
With the rapid proliferation of e-commerce, digital advertising has become
the predominant marketing channel across the industries: in 2019, businesses in US alone~\cite{news1} have spent more than 129 billion dollars--a number that has been fast growing and projected to increase--on digital ads, surpassing for the first time the combined amount spent via traditional advertising channels (TV, radio, newspapers, etc.), which falls short of 20 billion dollars. 
Situated in this background, online auctions--a core component of digital advertising--have become the most economically impactful element, both for publishers (entities that sell advertising spaces through auctions, a.k.a. sellers) and for advertisers (entities that buy advertising spaces through auctions to advertise, a.k.a. bidders). In practice, online advertising is implemented on platforms known as \emph{ad exchanges}, where the advertising spaces are sold through auctions between sellers and bidders.

In the past, due to its truthful nature that bidding one's private value is weakly dominant, the second-price auction\footnote{In a second-price auction, the highest bidder wins the auction but only pays the second-highest bid.} (possibly with reserve prices, also known as the Vickrey auction~\cite{vickrey1961counterspeculation}) was a popular auction mechanism and was almost universally adopted for online display ads auctions~\cite{lucking2000vickrey,klemperer2004auctions,lucking2007pennies}. 
However, very recently there has been an industry-wide shift from second-price auctions to first-price auctions in display ads\footnote{A wide range of ads, often made up of texts, images or video segments that encourage the user to click-through to a landing page and take some action (buy a product, use a service, sign up for a class etc).} auctions~\cite{despotakis2019first}, which account for 54\% of the digital advertising market share, a percentage that has seen continued growth ``fuled by the upswing in mobile browsing, social media activities, video ad formats, and the developments in targeting technology''~\cite{choi2020online}.  The remaining market share is dominated by search ads (e.g. sponsored search ads), which at this point are still exchanged between publishers and advertisers via generalized second-price auctions, although that could change in the future too.
See Figure \ref{fig:3}(a) for a schematic diagram.

Such a shift to first-price auctions has occured for several reasons:  enhanced transparency where the seller no longer has the ``last look'' advantage, an increased revenue of the seller and hence the exchange (which charges a percentage of the winning bid) and finally, fairness \cite{news2, Google}. To understand the last point on fairness, note that a seller would sometimes sell an advertising slot on different exchanges, and take the highest payment across the exchanges, which is known as \emph{header bidding}. Consequently, under second-price auctions, it is possible that a bidder who bids lower ends up winning the final bid, a fairness issue that would not happen under first-price auctions\footnote{This is not to say that the fairness issue under second-price auctions cannot be fixed. Technologically, it can be fixed easily. However, the bigger incentive (due to the reasons already mentioned) is for the exchanges to switch to first-price auctions, which also do not suffer from the fairness issue.}.

As a result of these advantages, several exchanges (e.g. AppNexus, Index Exchange and OpenX) started to roll out first-price auctions in 2017 \cite{Exchange}, and Google Ad Manager (previously known as Adx) completed its move to the first-price auctions\footnote{First-price auctions have also been
	the norm in several more traditional settings, including the mussels auctions~\cite{van2001sealed} (see~\cite{esponda2008information} for more discussion).} at the end of 2019 \cite{Google2}. 
This shift brings forth important challenges to bidders since the optimal bidding strategy in first-price auctions is no longer truthful. This thus leads to an important and pressing question, one that was absent in second-auctions prior to the shift: how should a bidder (adaptively) bid to maximize her cumulative payoffs when she needs to bid repeatedly facing a first-price auction?

Needless to say, the rich literature on auctions theory has studied related aspects of the problem. Broadly speaking, there are two major approaches that provide insights into bidding strategies in auctions. The first (and also the more traditional) approach takes a game-theoretic view by assuming a Bayesian setup, where the bidders have perfect or partial knowledge of each other's private valuations modeled as probability distributions. Proceeding from this standpoint, the pure or mixed (Nash) equilibria that model rational and optimal outcomes of the auction can be derived  \cite{wilson1969communications,myerson1981optimal,riley1981optimal}. Despite its elegance, an important shortcoming of this game-theoretic framework is that the participating bidders often do not have an accurate modeling of one's own value distributions. Consequently, these value distributions are even more unlikely to be known to other bidders or the seller in practice \cite{wilson1985game}. 

To mitigate this drawback, the second (and more recent) approach is based on online learning in repeated auctions, where the participants can learn their own or others' value distributions over time. Under this framework, a flourshing line of literature studies the second-price auction, mostly from the seller's perspective who aims for an optimal reserve price \cite{medina2014learning,cesa2014regret,roughgarden2019minimizing,zhao2020online}. There are also a few papers that take the bidder's perspective in the second-price auction \cite{mcafee2011design,weed2016online}, where the bidder does not have a perfect knowledge of her own valuations. 

However, to date, with the exception of the pioneering work~\cite{balseiro2019contextual} (which has adopted a cross-learning approach discussed in more detail in Section~\ref{subsec:related}), the problem of learning to bid in repeated first-price auctions has not yet been adequately addressed. 
In fact, not only is this not yet rigorously studied in an intellectual framework, but it also appears that no effective heuristics have been developed satisfactorially by the bona fide bidders in the industry. As documented in a recent report by the ad exchange AppNexus in 2018, ``the available evidence suggests that many large buyers have yet to adjust their bidding behavior for first-price auctions''~\cite{news7}, and as a result, after the transition to first-price occured, the bidders' spending increased substantially; see Figure~\ref{fig:3}(b). Similar empirical findings were also obtained in a recent study \cite{goke2021learning}. Consequently, developing methologically sound bidding schemes in first-price auctions presents an important research opportunity.

In this paper, we aim to tackle this problem and establish the optimal bidding strategy which minimizes the bidder's regret. Specifically, we consider learning in repeated first-price auctions where only the transaction price (i.e. the winning bid) is observed by each bidder and where the private values of the bidder may vary over time. This feedback structure is a valid feedback in practice (e.g. in equity markets where the transaction price of a stock is announced in real time), and more importantly, leads to an interesting and sophisticated scenario of \emph{censored feedback} in theory: the \emph{winner} cannot observe the highest other bid (HOB) on that round. In order for the bidder to learn from the history, we assume that the highest bids of others are stochastic and follow an unknown \textit{iid} distribution. In our online learning setting, the bidder is competing against a strong oracle, one that knows the underlying HOB distribution and hence can bid optimally at each time.  More discussions on these assumptions will be placed in the problem formulation. 

\begin{figure}[t!]
	\centering
	\begin{subfigure}[t]{0.45\textwidth}
		\centering
		\includegraphics[height=4.5cm]{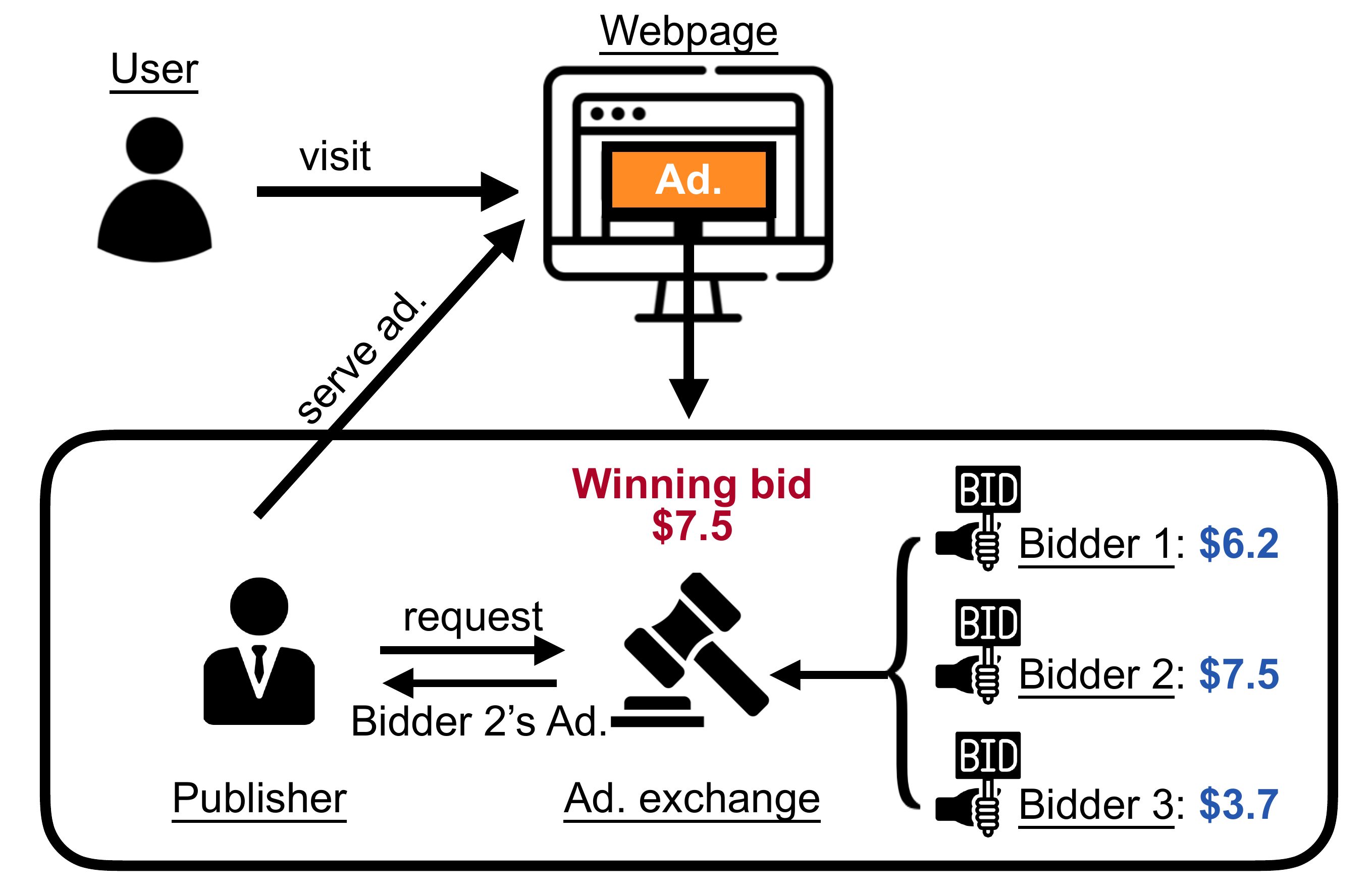} 
		\caption{} \label{fig:1}
	\end{subfigure}
	\begin{subfigure}[t]{0.52\textwidth}
		\centering
		\includegraphics[height=4.5cm]{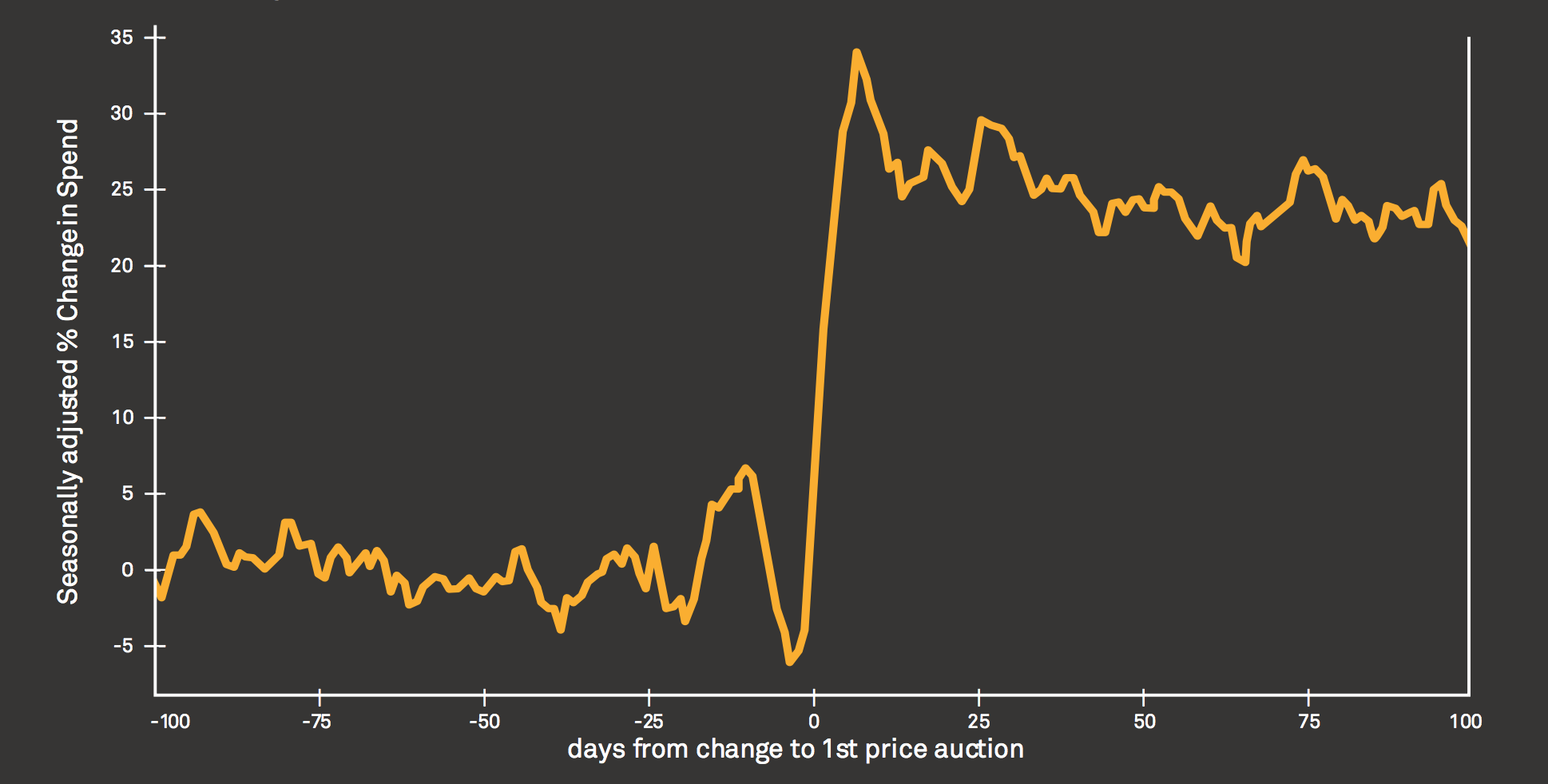} 
		\caption{} \label{fig:2}
	\end{subfigure}
	\hfill
	\caption{
		Panel (a) provides a simplified schematic diagram first-price display ads auction. 
		Panel (b) shows the spending change for bidders' on AppNexus~\cite{news7}, which switched from second-price to first-price auctions in 2018. } \label{fig:3}
\end{figure}

\subsection{Key Challenge}

The key challenge in this problem lies in censored feedback and its impact on the learning process: if the bidder bids a good price that wins, she will not learn anything about the HOB. This is the \textit{curse} of censored feedback, which presents an exploration-exploitation trade-off that is distinct from and more challenging than that of the standard contextual bandits (with bandit feedback).

To appreciate the difficulty of the problem, we cast the first-price auction as a contextual bandit where the bids are the \emph{actions}, and the bidder's private values are the \emph{contexts}. The censored feedback in first-price auctions leads to a feedback structure across both actions and contexts, discussed in Section \ref{subsec.results}. Under this feedback structure, with $K$ actions and $M$ contexts, the idea of using graph feedback across actions in \cite{alon2015online,cohen2016online,alon2017nonstochastic,lykouris2020feedback} achieves $\widetilde{O}(\sqrt{MT})$ regret over a time horizon $T$, and the idea of exploiting cross learning across contexts in \cite{balseiro2019contextual,dann2020reinforcement} leads to $\widetilde{O}(\sqrt{KT})$ regret. Both results are much larger than $\widetilde{O}(\sqrt{T})$, as $K$ and $M$ grow with $T$ in the discretization.

Another type of algorithm making use of the censored feedback is called explore-then-commit (ETC), where the bidder intentionally loses the first $T_0$ rounds to estimate the HOB distribution, and performs a pure exploitation in the remaining rounds. It is easy to see that the optimal choice of $T_0$ is $\Theta(T^{1/3})$, and the ETC algorithm achieves an $\widetilde{O}(T^{2/3})$ regret; this is still larger than $\widetilde{O}(\sqrt{T})$. 


Here is the key contribution of this paper: we show that an $\widetilde{O}(\sqrt{T})$ regret is possible for first-price auctions, by identifying additional structures on the feedback and payoff. As summarized in Section \ref{subsec.results}, our first algorithm requires a novel partial-ordering property possessed by first-price auctions which provides additional feedback, and our second algorithm crucially utilizes a correlation structure which is unique in the payoff function of first-price auctions. Furthermore, we show that relying \emph{only} on the feedback structure does not lead to an $\widetilde{O}(\sqrt{T})$ regret for first-price auctions in general - we prove an $\Omega(\min\{\sqrt{MT}, \sqrt{KT}, T^{2/3}\})$ regret lower bound which shows the optimality of the above algorithms in the literature. In other words, achieving the optimal learning performance in first-price auctions requires \emph{both} the development of general bandit algorithms, \emph{and} an in-depth exploitation of the specific structures in first-price auctions.

\subsection{Our Algorithms and Results}\label{subsec.results}
In this paper, we provide two algorithms which achieve the near-optimal $\widetilde{O}(\sqrt{T})$ regret in first-price auctions with stochastic and adversarial private values, respectively. Although the former setting is subsumed by the latter, the algorithm developed for the former is entirely different, and has deeper connections to --- and thus enriches --- the bandits literature. 

\subsubsection{Stochastic private values}
The first algorithm works under stochastic private values and mainly exploits the censored feedback structures in the auction:
\emph{\begin{enumerate}
		\item Once we know the payoff under a certain bid, then irrespective of whether this bid wins or not, we know the payoff of \emph{any} larger bid; 
		\item Once we know the payoff of a bid under a certain private value, then we know the payoff of that bid under \emph{any} private values. 
\end{enumerate}}
In the bandits language, the bids are \emph{actions}, and the private values are \emph{contexts}. The above observations state that, bidding in first-price auctions is a contextual bandit with a one-sided feedback across actions and a full-information feedback across contexts. 
In the bandits literature, there has been studies on each of the above-mentioned 
two feedback structures under the names of graph feedback \cite{alon2015online,cohen2016online,alon2017nonstochastic,lykouris2020feedback} and cross learning \cite{balseiro2019contextual}, respectively. 
The joint study of both structures is also available in a concurrent work \cite{dann2020reinforcement}. However, these results do not lead to the desired $\widetilde{O}(\sqrt{T})$ regret in first-price auctions: the above two observations are insufficient without the third: 
\emph{\begin{enumerate}
		\item[3.] The optimal bid never decreases when the private value increases. 
\end{enumerate}}
The third observation has no direct counterpart in the bandits literature, and in this paper we state a general version of this observation and formulate a new bandit instance called the \emph{partially ordered contextual bandits}. Our first algorithm is a general algorithm for partially ordered contextual bandits, and leads to an $O(\sqrt{T}\log^{2.5} T)$ regret (Theorem \ref{thm.stochastic}, almost matching the lower bound $\Omega(\sqrt{T})$ in Theorem \ref{thm:sqrtT_lower_bound}) when specialized to first-price auctions with \emph{iid} private values. 

For both pedagogical clarity and articulation of its novelty and connections with the existing work, we build up our algorithm in several layers, each with increasing complexity:

\begin{itemize}
	\item First, we consider stochastic $K$-armed bandits with graph feedback across actions (pulling one arm reveals the rewards of not only this arm, but also arms that form edges with this arm) and present a simple arm-elimination algorithm (Algorithm \ref{algo.MAB}) that achieves the optimal $\widetilde{\Theta}(\sqrt{\alpha T})$ regret where $\alpha$ is the independence number of the feedback graph (cf. Theorem \ref{thm.MAB_feedback}). 
	\item Second, we consider contextual bandits with cross learning across contexts and graph feedback across actions, where taking one action under a context reveals the rewards of all context-action pairs whose actions form edges with this action in the feedback graph. We show that a simple generalization of our non-contextual algorithm achieves an $\widetilde{O}(\sqrt{\min\{\alpha M, K \}T})$ regret bound, without any assumption on how the context is generated, where $K$ is the number of actions and $M$ is the number of contexts (cf. Algorithm \ref{algo.contextual} and Theorem \ref{thm.contextual}). This recovers both results of $\widetilde{O}(\sqrt{\alpha MT})$ regret with only graph feedback in \cite{alon2015online} and $\widetilde{O}(\sqrt{KT})$ regret with only cross learning in \cite{balseiro2019contextual}, by setting $M=1$ and $\alpha=K$, respectively. 
	
	\item Third, we impose an additional partial ordering structure over contexts to form a partially ordered contextual bandit, and propose our main algorithm (Algorithm \ref{algo.partial_order}) that achieves an $\widetilde{O}(\sqrt{\alpha\beta T})$ regret bound when the contexts are \emph{iid} generated, where $\beta$ is the independence number of the partial ordering graph (cf. Theorem \ref{thm:partial_order}). This is a notable improvement over the $\widetilde{O}(\sqrt{\alpha M T})$ regret achieved above, and is the key to obtaining an $\widetilde{O}(\sqrt{T})$ regret for first-price auctions (with stochastic private values), in which case $\alpha=\beta=1$ (cf. Corollary \ref{cor.stochastic}). 
	\item Finally, we establish a fundamental learning limit for partially ordered contextual bandits, and in particular, a curious separation between stochastic and adversarial contexts. Specifically, we show that if the contexts are generated by an oblivious adversary, even when $\beta = 1$, any learning algorithm must suffer from an $\Omega(\min\{\sqrt{\alpha MT}, \sqrt{KT}, T^{2/3}\})$ regret lower bound in the worst case (cf. Theorem \ref{thm.lower_bound}), which is minimax optimal (up to log factors) when $\alpha=1$ (cf. Corollary \ref{cor.minimax_regret}). The implication of this minimax optimal bound $\widetilde{\Theta}(\min\{\sqrt{MT}, \sqrt{KT}, T^{2/3} \})$ is profound for first-price auctions (where $\alpha = \beta =1$): when the private values are advesarially generated, then casting first-price auctions as a partially ordered contextual bandits problem yields at best a $\widetilde{\Theta}(T^{2/3})$ regret bound (see Remark~\ref{rem:lower_bound_insufficient}). As such, if $\widetilde{O}(\sqrt{T})$ regret is indeed achieveable for first-price auctions under adversarial private values, then the abstraction provided by partially ordered contextual bandits is not fine-grained enough to capture \textit{all} the useful properties of a first-price auction, in which case a new framework is called for.
\end{itemize}

\subsubsection{Adversarial private values}
Surprisingly, the $\widetilde{O}(\sqrt{T})$ regret is indeed achieveable for first-price auctions under adversarial private values, and our second algorithm (which no longer operates in the previous framework) achieves it. In fact, a closer look reveals that the above framework makes no assumption on the payoffs of different actions (except for taking value in $[0,1]$).
In contrast, the crux of our second algorithm makes use of the special payoff structure of first-price auctions and comes from the following observation which complements the previous observation 1: 
\emph{\begin{enumerate}
		\item[4.] once we know the payoff under a bid, then we \emph{partially} know the payoff of any smaller bid. 
\end{enumerate}}
This observation is a consequence of the correlated payoff structure which is unique in first-price auctions, and the formal meaning of ``partial'' is detailed in Section 4. Roughly speaking, the payoff function in first-price auctions crucially depends on the CDF of the HOB distribution, which could be written as a sum of several pieces --- the bidder has different amounts of information for different pieces. Building on this insight, we are able to provide an interval-splitting scheme that estimates the unknown HOB distribution via a dynamic partition scheme: the CDF of the HOB distribution is estimated on an appropriately chosen set of partitioning intervals, where each interval has a certain sample size that is just sufficient to estimate the probability of that interval to the required accuracy.
Further, on top of the previous insights, we develop a master algorithm which decouples the exploration and exploitation in an efficient way to help handle the complicated dependence in sequential learning. 
Putting all these pieces together yields the learning algorithm ML-IS-UCB (Algorithm~\ref{algo.ml-ucb}), which
achieves $O(\sqrt{T}\log^3 T)$ regret (Theorem~\ref{thm.arbitrary}) and is hence minimax optimal (up to $\log$ factors) as a result of the $\Omega(\sqrt{T})$ lower bound given in Theorem~\ref{thm:sqrtT_lower_bound}.


\subsection{Related Work}\label{subsec:related}

Modeling repeated auctions as bandits has a long history~\cite{blum2004online,devanur2009price,babaioff2014characterizing,medina2014learning,cesa2014regret,babaioff2015truthful,mohri2015revenue,weed2016online,cai2017learning,golrezaei2019dynamic,roughgarden2019minimizing,zhao2020online} that lies at the intersection between learning and auctions.  In these work, the auctions are typically modeled as multi-armed bandits without any contexts, where the competing oracle can only choose a fixed action. At the same time, some work considered the censored feedback structure in different auction settings. 
In \cite{cesa2014regret}, the seller can observe the second highest bid only if she sets a reserve price below it in the second-price auction. In \cite{weed2016online}, the bidder in the second-price auction does not know her own private valuation of a good and must learn to bid when simultaneously learning that valuation. In particular, the bidder can update her private valuation of the good only when she makes the highest bid and gets the good. 
Another important work in the unknown valuation setting is~\cite{10.1145/3219166.3219208}, where a general auction mechanism is considered. The work \cite{10.1145/3219166.3219208} differs from ours in several key aspects. First, we assume unknown HOB and known private valuation, while \cite{10.1145/3219166.3219208} assumes known HOB but aims to learn private valuations. Second, \cite{10.1145/3219166.3219208} assumes a different feedback model that the bidder observes full information only if she wins. 
Third, the vanishing regret of the online learning algorithm of \cite{10.1145/3219166.3219208}
is with respect to the best fixed bid in hindsight, whereas ours competes with a bidding policy that varies the bids depending on the private values. 
In \cite{zhao2019stochastic} and \cite{zhao2020online}, the authors proposed a one-sided full-information structure in stochastic bandits, with applications in the second-price auction. This one-sided feedback structure is similar to ours, but their work did not consider the contextual case. 

Beyond the auction setting, there is a rich line of research in the bandit problem with general feedback structures (also known as \emph{partial monitoring}). We discuss how our results for partially ordered contextual bandits are related to existing ones in the literature. For multi-armed bandits without contexts, it was shown in \cite{kocak2014efficient,alon2015online,alon2017nonstochastic} that the optimal regret under graph feedback is $\widetilde{\Theta}(\sqrt{\alpha T})$, where $\alpha$ is the independence number of the feedback graph. Under the special case of stochastic bandits, simpler algorithms based on arm elimination were also proposed in \cite{cohen2016online,lykouris2020feedback}. Our arm elimination based algorithm is of a similar spirit, but easier to generalize to the contextual case (more detailed comparisons on the arm elimination algorithms are discussed under Algorithm \ref{algo.MAB}). The setting of contextual bandits with both cross learning and graph feedback is a combination of action graph feedback considered in~\cite{alon2015online} and cross learning considered in~\cite{balseiro2019contextual}. At the action level, information flows according to the given feedback graph, while at the context level, information flows everywhere according to the complete graph. Our regret bound $\widetilde{O}(\sqrt{\min\{\alpha M, K\}T})$ recovers both results of $\widetilde{O}(\sqrt{\alpha MT})$ regret with only graph feedback in \cite{alon2015online} and $\widetilde{O}(\sqrt{KT})$ regret with only cross learning in \cite{balseiro2019contextual}, by setting $M=1$ and $\alpha=K$, respectively. For additional work on partial feedback with contexts, the paper \cite{cesa2017algorithmic} studied a one-sided full-information feedback structure under the setting of online nonparametric learning and proved a $\widetilde{\Theta}(T^{2/3})$ bound on the minimax regret. The recent work \cite{dann2020reinforcement} considered a general feedback graph over the context-action pair in reinforcement learning and proved an $\widetilde{O}(\sqrt{mT})$ regret bound, where $m$ is the size of the maximum acyclic subgraph. When the feedback graph over actions is acyclic (as is the case for first-price auctions), it holds that $m=K$ and this result reduces to the $\widetilde{O}(\sqrt{KT})$ upper bound in \cite{balseiro2019contextual}. Consequently, the results in \cite{balseiro2019contextual,dann2020reinforcement} did not lead to the $\widetilde{O}(\sqrt{\alpha\beta T})$ regret for partially ordered contextual bandits, nor the separation between stochastic and adversarial contexts. 

We also review the literature on censored observations in pricing or auctions. In dynamic pricing, a binary feedback indicating whether an item is sold to a particular buyer was studied in \cite{kleinberg2003value,cesa2018dynamic}, where they derived the optimal regret using multi-armed bandits. In particular, their results imply that if only a binary feedback is revealed in the first-price auction (i.e. whether the bidder wins or not), even under an identical private value of the bidder over time, the worst-case regret is at least $\Omega(T^{2/3})$. 
The same $\widetilde{\Theta}(T^{2/3})$ regret bound was also shown in \cite{balseiro2019contextual} for learning in repeated first-price auctions with binary feedback, where the bidder only knows whether she wins the bid or not. In comparison, our work shows a $\widetilde{\Theta}(\sqrt{T})$ regret when the winning bid is additionally available, thereby adding an interesting piece to the picture of learning in repeated first-price auctions. Moreover, the regret bounds of \cite{kleinberg2003value,cesa2018dynamic,balseiro2019contextual} were obtained by only using the feedback structures in the respective bandit problems. In contrast, our setting of first-price auctions under arbitrary/adversarial private values is not amenable to either cross-learning or partially ordered contextual bandits, and we need to identify and exploit additional structures in first-price auctions to design an algorithm that achieves the optimal $\widetilde{O}(\sqrt{T})$ regret. 

%% file: rest.tex
\section{Problem Formulation and Main Results}\label{sec.formulation}
In this section, we formulate our learning problem in repeated first-price auctions, with discussions on various assumptions. The main results of this paper are stated subsequently. 

\subsection{Notation}\label{subsec:notation}
For a positive integer $n$, let $[n]\triangleq \{1,2,\cdots,n\}$. For an $x\in\bR$, let $\lceil x \rceil$ be the smallest integer no smaller than $x$. For $x,y\in \bR$, let $x\wedge y=\min\{x,y\}$ be the minimum. For a square-integrable random variable $X$, let $\bE[X]$ and $\var(X)$ be the expectation and variance of $X$, respectively. For an event $A$, let $\1(A)$ be the indicator function of $A$ which is one if $A$ occurs and zero otherwise. For probability measures $P$ and $Q$ on the same probability space, let $\|P-Q\|_{\text{TV}} = \frac{1}{2}\int|dP-dQ|$ and $D_{\text{KL}}(P\|Q) = \int dP\log\frac{dP}{dQ}$ be the total variation distance and KL-divergence between $P$ and $Q$, respectively. We also adopt the standard asymptotic notation: for non-negative sequences $\{a_n\}$ and $\{b_n\}$, $a_n = O(b_n)$ if $\limsup_{n\to\infty} a_n/b_n < \infty$, $a_n = \Omega(b_n)$ if $b_n = O(a_n)$, and $a_n = \Theta(b_n)$ if both $a_n = O(b_n)$ and $a_n = \Omega(b_n)$. We also use $\widetilde{O}(\cdot), \widetilde{\Omega}(\cdot), \widetilde{\Theta}(\cdot)$ to denote the respective meanings above within multiplicative poly-logarithmic factors in $n$.

\subsection{Problem Setup}\label{subsec:setup} 
We focus on a single bidder in a population of bidders during a time horizon $T$. At the beginning of each time $t=1,2,\cdots,T$, the bidder sees a particular good and receives a private value $v_t\in [0,1]$ for this good. Based on her past observations, the bidder bids a price $b_t\in [0,1]$ for this good. Let $m_t\in [0,1]$ be the highest other bid (HOB), i.e. the maximum bid of all other bidders. The outcome of the auction is as follows: if $b_t \ge m_t$, the bidder gets the good and pays $b_t$; if $b_t<m_t$, the bidder does not get the good and pays nothing\footnote{By a slight perturbation, we assume without loss of generality that the bids are never equal.}. Consequently, the instantaneous payoff/reward of the bidder is
\begin{align}\label{eq.instant_reward}
	r(v_t,b_t;m_t) = (v_t - b_t)\1(b_t \ge m_t). 
\end{align}

We assume that only the winning bid is revealed at the end of time $t$ (hence the bidder \textit{only} observes $m_t$ if she loses). This can be viewed as an informational version of the \emph{winner's curse} \cite{capen1971competitive} where the winner has less information to learn from, and the feedback available to the bidder at time $t$ is $\max\{b_t, m_t\}$, or equivalently, the quantities $\1(b_t\ge m_t)$ and $m_t\1(b_t<m_t)$. The above structure holds in a number of practical first-price auctions where only the final transaction price is announced \cite{esponda2008information}. 
Compared with the full-information feedback (i.e. $m_t$ is always revealed) and binary feedback (i.e. only $\1(b_t\ge m_t)$ is revealed), this structure is technically more challenging to exploit and leads to various technical findings. We also note that this setting is symmetric to the case where the HOB is only revealed to the winner, i.e. bidder's feedback is $\min\{b_t, m_t\}$ instead of $\max\{b_t, m_t\}$; see Section \ref{subsec.reversed_feedback} for more discussions. 

To model other bidders' bids, we assume $m_t$'s are \emph{iid} drawn from an unknown cumulative distribution function (CDF) $G(\cdot)$, with $G(x) = \bP(m_t\le x)$. The main rationale behind this assumption is that there is potentially a large population of bidders, and on average their valuations and bidding strategies are static over time, and in particular, independent of a single bidder's private valuation. Moreover, the \emph{iid} assumption makes learning possible for the bidder, so that in the sequel she can compete against a strong oracle. Note that we do not make any additional smoothness or shape assumptions on $G(\cdot)$, e.g. $m_t$ could either be discrete or continuous. This assumption is also adopted by many demand side platforms where stationarity of data is assumed for a certain time window. Moreover, learning could become impossible even when the \emph{iid} assumption is only slightly violated; see Section \ref{subsec.lower_bound} for more details. 

We consider two models for the private values $v_t$. One is that $(v_t)_{t\in [T]}$ are stochastic and \emph{iid} drawn from some unknown distribution $F$. The other is that $(v_t)_{t\in [T]}$ is an adversarial sequence taking values in $[0,1]$, chosen by an \emph{oblivious adversary}, who may choose $v_t$ in an arbitrary way--without knowing and hence are independent of the bids $(b_t,m_t)_{t\in [T]}$--as the private values in practice are usually learned using other sources of information and independently of the bidding process. We recall that $v_t$ is always revealed to the bidder, regardless of how it is generated.  

As such, the expected reward of the bidder at time $t$ is
\begin{align}\label{eq.expected_reward}
	R(v_t,b_t) = \bE[r(v_t,b_t;m_t)] = (v_t - b_t)G(b_t). 
\end{align}
The regret of the bidder is defined to be the difference in the cumulative rewards of the bidder's bidding strategy and the optimal bidding strategy which has the perfect knowledge of $G(\cdot)$:
\begin{align}\label{eq.regret}
	R_T(\pi; v) = \sum_{t=1}^T \left(\max_{b\in [0,1]}R(v_t,b) - R(v_t,b_t) \right),
\end{align}
where $\pi$ is the overall bidding policy that generated $b_t$'s, and $v=(v_1,\cdots,v_T)$ is a given sequence of private values. The bidder's goal is to devise a policy $\pi$ minimizing the expected regret $\bE[R_T(\pi; v)]$ subject to the censored feedback structure. When $v$ is stochastic, the expectation $\bE[R_T(\pi;v)]$ in \eqref{eq.regret} is taken jointly over the randomness of both $\pi$ and $v$. In contrast, when $v$ is chosen adversarily, the bidder aims to achieve a uniformly small expected regret $\bE[R_T(\pi;v)]$ regardless of $v$. 

\subsection{Main Results}\label{subsec:main}

Our first result shows that, under stochastic private values, an $\widetilde{O}(\sqrt{T})$ expected regret is attainable. 
\begin{theorem}[Bidding with stochastic private values]\label{thm.stochastic}
	Let $v_1,\cdots,v_T$ be iid drawn from any unknown distribution $F$. Then there exists a bidding policy $\pi$ (Algorithm \ref{algo.partial_order} applied to first-price auctions, see Corollary \ref{cor.stochastic}) satisfying
	\begin{align*}
		\bE[R_T(\pi;v)] \le C\sqrt{T}\log^{2.5} T, 
	\end{align*}
	where the expectation is taken jointly over the randomness of $v$ and the policy $\pi$, and $C>0$ is an absolute constant independent of the time horizon $T$ and the CDFs $(F,G)$. 
\end{theorem}

Theorem \ref{thm.stochastic} is obtained as a special case to the regret bound on a more general class of problems, called the partially ordered contextual bandits (see Definition~\ref{assump:partial_order}).
However, we also show that if one works in this general class of problems, then the \emph{iid} condition of $v_t$ in Theorem \ref{thm.stochastic} is necessary for achieving an $\widetilde{O}(\sqrt{T})$ regret (see Remark~\ref{rem:lower_bound_insufficient} in Section \ref{subsec.lowerbound}). It turns out that this general class of problems does not capture certain specific structure of first-price auctions. Hence, in our second result, we propose another bidding strategy tailored specifically for the first-price auctions which achieves an $\widetilde{O}(\sqrt{T})$ regret (with a slightly worse logarithmic factor) under adversarial private values. 

\begin{theorem}[Bidding with arbitrary private values]\label{thm.arbitrary}
	Let $v_1,\cdots,v_T$ be any value sequence in $[0,1]$ which may be chosen by an oblivious adversary. Then there exists a bidding policy $\pi$ (the \text{\rm ML-IS-UCB} bidding policy in Section \ref{subsec.coUCB}) with
	\begin{align*}
		\bE[R_T(\pi;v)] \le C\sqrt{T}\log^3 T, 
	\end{align*}
	where the expectation is taken with respect to the randomness of $\pi$, and $C>0$ is an absolute constant independent of the time horizon $T$, the value sequence $v$ and the unknown CDF $G$. 
\end{theorem}

To the best of our knowledge, the above theorems present the first $\widetilde{O}(\sqrt{T})$ regret bidding strategies for general unknown CDF $G$ in the first-price auction. The $\Omega(\sqrt{T})$ regret lower bound even under full-information where $m_t$ is always revealed at each time can be established using the standard Le Cam's two-point method; for completeness we include the proof in Appendix \ref{appendix.lower_bound}. 

\subsection{Organization}

The rest of the paper is organized as follows. In Section \ref{sec.stochastic}, we present the framework of partially ordered contextual bandits, and prove a general result which subsumes Theorem \ref{thm.stochastic}. Section \ref{sec.adversarial} is devoted to the proof of Theorem \ref{thm.arbitrary}, and Section \ref{sec:discussion} contains additional extensions and discussions of the current setting. The proofs of many auxiliary results are relegated to the appendix. 

\section{Learning in First-price Auctions with Stochastic Private Values}\label{sec.stochastic}
In this section, we formulate repeated first-price auctions as a particular type of contextual bandits with structured feedback, and establish general regret bounds. Specifically, we consider a general contextual bandit problem combining the feedback graph structure across actions in \cite{alon2015online,alon2017nonstochastic} and the cross learning structure across contexts in \cite{balseiro2019contextual}. A natural question is the characterization of the optimal regret in this setting. Several upper bounds directly follow from the literature: with $K$ actions and $M$ contexts, by only exploiting the feedback structure across actions and studying $M$ independent bandit problems, a regret upper bound $\widetilde{O}(\sqrt{\alpha MT})$ is achievable, where $\alpha$ is the independence number of the feedback graph across actions; by only exploiting the cross learning structure across contexts, an $\widetilde{O}(\sqrt{KT})$ regret is available in \cite{balseiro2019contextual}. We show, by first proposing an algorithm for multi-armed bandits and then generalizing it to contextual bandits, that an $\widetilde{O}(\sqrt{\min\{\alpha M, K \}T})$ regret can be achieved --- here the effects of actions and contexts are decoupled. To further reduce the regret, we show that if there is an additional partial ordering structure over the contexts with independence number $\beta$, \emph{and} if the contexts are \emph{stochastic}, a smaller regret $\widetilde{O}(\sqrt{\alpha\beta T})$ is achievable --- here the effects of actions and contexts are coupled. This leads to the $\widetilde{O}(\sqrt{T})$ regret in Theorem \ref{thm.stochastic} for first-price auctions with stochastic private values, where $\alpha=\beta=1$. A curious phenomenon is that, these effects are fundamentally decoupled for adversarial contexts: we prove a regret lower bound $\Omega(\sqrt{\min\{\alpha M,K,T^{1/3}\}T})$, showing a distinction between stochastic and adversarial contexts. This is the interesting interplay between actions and contexts we aim to highlight in Section \ref{sec.stochastic}.

This section is organized as follows. Section \ref{subsec.MAB_feedback} presents a simple algorithm for stochastic multi-armed bandits which achieves the optimal $\widetilde{O}(\sqrt{\alpha T})$ regret with feedback graphs across actions, and Section \ref{subsec.contextual_feedback} extends it to contextual bandits with cross learning across contexts. In Section \ref{subsec.partial_order}, we impose an additional partial ordering structure over the contexts (termed as \emph{partially ordered contextual bandits}), where a variant of the previous algorithm achieves a regret almost independent of $(K, M)$ if the contexts are stochastic. In particular, this general result leads to an $\widetilde{O}(\sqrt{T})$ regret for first-price auctions with stochastic private values in Section \ref{subsec.application_auction}. Finally Section \ref{subsec.lowerbound} proves the limitation of this framework: without the assumption of stochastic contexts, a regret lower bound $\Omega(\sqrt{T})$ is unavoidable for the general class of partially ordered contextual bandits.

\subsection{An Algorithm for Stochastic Multi-armed Bandits with Graph Feedback}\label{subsec.MAB_feedback}

We recall the following setting of stochastic multi-armed bandits with a feedback graph $G$, introduced in \cite{alon2015online}. 
\begin{definition}[Stochastic Multi-armed Bandits with Graph Feedback]\label{def.MAB_feedback}
	
	A stochastic multi-armed bandit with graph feedback consists of a finite time horizon $T$, an action space $\calA$ with $|\calA| = K$, a joint reward distribution $P$ supported on $[0,1]^K$ with marginal distributions $(P_a)_{a\in \calA}$ and mean rewards $(R_a)_{a\in \calA}$, and a directed graph $G = (\calA, E)$ with vertex set $\calA$ and edge set $E$.
	\begin{enumerate}
		\item \text{\rm \bf Learning and Feedback Model.}
		At each $t\in [T]$, the learner chooses an action $a_t\in \calA$ based on her historical observations. Nature then draws a reward vector $r_t = (r_{t,a})_{a\in \calA}\overset{iid}{\sim} P$ and reveals $(r_{t,a})_{a\in I_t}$ to the learner, where $I_t = \{a_t\} \cup \{a\in \calA: (a_t, a) \in E \}$. 
		\item \text{\rm \bf Regret.} 
		Let $\pi = (a_1, a_2, \dots, a_T)$ be the policy used by the learner, then the regret is:
		\begin{align}\label{eq:MAB_regret}
			R_T(\pi) \triangleq \sum_{t=1}^T \left(\max_{a\in\calA} R_{a} - R_{a_t} \right). 
		\end{align}
	\end{enumerate}
\end{definition} 

The directed graph $G$ defines a feedback structure in the following sense: if an action $a$ is played, the rewards of action $a$ and all its out-neighbors in $G$ are revealed to the learner. Note that we manually take out all self loops and assume that $G$ no longer contains any self loop. It was known in \cite{alon2015online,alon2017nonstochastic} that for adversarial bandits, the optimal regret is $\widetilde{\Theta}(\sqrt{\alpha T})$, where $\alpha$ is the independence number of $G$, i.e. the largest size $m$ of a vertex set $\{a_1, \cdots, a_m\}$ such that $(a_i, a_j)\notin E$ for all $i\neq j$. Similar results were also obtained in \cite{cohen2016online,lykouris2020feedback} for stochastic bandits. Nevertheless, we will present a simpler algorithm for stochastic bandits with graph feedback which is easy to generalize to the contextual setting. 

Although the results in this section hold for any directed graph $G$, in the future we will also have a special focus on the case that $G$ is a \emph{directed acyclic graph (DAG)}, i.e. $G$ contains no cycles. The following Algorithm \ref{algo.MAB}, based on successive arm elimination, will serve as a building block for all subsequent algorithms in Section \ref{sec.stochastic}. There are two possible options for choosing the action $a_t$, where Option I works for any directed graph $G$, and Option II is a more convenient choice for DAGs. 

\begin{algorithm}[h!]
	\caption{Arm Elimination for Stochastic Multi-armed Bandits with Graph Feedback}\label{algo.MAB}
	\textbf{Input:} Time horizon $T$; action set $\calA = [K]$; feedback graph $G$; failure probability $\delta\in (0,1)$. \\
	\textbf{Output:} A resulting policy $\pi$. \\
	\textbf{Initialization:}
	Active set $\calA_{\text{act}} \gets \calA$, minimum count $N\gets 0$; 
	
	All (empirical reward, action count) pairs $(\bar{r}_{a}^0, n_{a}^0)$ initialized to $(0,0)$ for each $a\in \calA.$\\
	\For{$t\in\{1,2,\cdots,T\}$}{
		Let $\calA_0 = \{a\in \calA_{\text{act}}: n_a^{t-1} = N\}$ be the set of active actions which have been pulled the fewest times, and $G_0$ be the restriction of $G$ to the vertex set $\calA_0$;\\
		The learner chooses the following action (break ties arbitrarily): \\
		\quad \textbf{(Option I, general graph)} $a_t \gets $ action in $\calA_0$ with the largest out-degree in $G_0$;\\
		\quad \textbf{(Option II, DAG)} $a_t \gets$ action in $\calA_0$ with in-degree zero in $G_0$;\\
		The learner receives random rewards $(r_{t,a})_{a\in I_t}$ with $I_t = \{a_t\}\cup \{a\in \calA: (a_t, a)\in E \}$; \\
		The learner updates the empirical rewards and action counts $(\bar{r}_a^t, n_a^t)$ for $a\in I_t$; \\
		\If{$\min_{a\in \calA_{\text{\rm act}}} n_a^t > N$}{
			Update the minimum count $N\gets N+1$; \\
			Let $\bar{r}_{\max}^t=\max_{a\in \calA_{\text{act}}} \bar{r}_{a}^t$ be the maximum empirical reward among active actions; \\
			Update the active set
			\begin{align*}
				\calA_{\text{act}} \gets \left\{a\in \calA_{\text{act}}: \bar{r}_a^t \ge \bar{r}_{\max}^t - 6\sqrt{\frac{\log(2T)\log(KT/\delta)}{N}} \right\}. 
			\end{align*}
		}
	}
\end{algorithm}

Algorithm \ref{algo.MAB} falls into the category of arm elimination policies \cite{even2006action}, where the learner maintains an active set of actions which contains all probably good actions and shrinks over time as more observations are collected. In Algorithm \ref{algo.MAB}, the active set $\calA_{\text{act}}$ is initialized to be the entire action set $\calA$, and is updated whenever the rewards of all active arms have been observed once, twice, and so on. In other words, if we denote by $N$ the minimum number of times the rewards of all active arms have been observed, the active set is updated whenever the value of $N$ changes. The construction of $\calA_{\text{act}}$ is based on the usual principle of confidence bounds, and the key algorithmic difference lies in the sampling rule between two successive updates of the active set under the graph feedback. Without the graph feedback, each active action is simply pulled once; with the graph feedback, we use a sampling rule which favors actions giving the most amount of information. Recall the definition of $N$ that all active arms have been pulled for at least $N$ times, and we aim to increase this quantity from $N$ to $N+1$ between two successive updates of the active set. To this end, we consider the set $\calA_0$ which has \emph{only} been pulled for $N$ times, and choose one of the following: 
\begin{enumerate}
	\item Option I: pull any action with the largest out-degree in $\calA_0$. This choice is valid for any directed graph $G$.
	\item Option II: pull any action with no parent in $\calA_0$. This choice is valid only for $G$ being a DAG. 
\end{enumerate}
Roughly speaking, both options aim to choose the most informative action, and we will show that $\widetilde{O}(\alpha)$ pulls increases the quantity of $N$ by one. 

We note that the algorithms in \cite{cohen2016online,lykouris2020feedback} also relied on arm elimination, with the following differences. First, the setting of \cite{lykouris2020feedback} is restricted to undirected graphs, while the regret bound in Theorem \ref{thm.MAB_feedback} (with Option I) applies to general directed graphs. Second, the algorithm in \cite{cohen2016online} is randomized and chooses a uniformly random action in the active set (due to the changing feedback graph over time), while \cite{lykouris2020feedback} needs to find a maximal independence set in each epoch. In contrast, Algorithm \ref{algo.MAB} is deterministic and simply needs to find a vertex with the highest degree. Our choice is computationally cheaper, easier to generalize to the contextual setting in Section \ref{subsec.contextual_feedback}, while makes the analysis slightly more involved and require a non-trivial result (cf. Lemma \ref{lemma.independence_number}) in combinatorial graph theory. 

The following theorem shows that the above algorithm achieves an $\widetilde{O}(\sqrt{\alpha T})$ regret, known to be optimal (within logarithmic factors) in \cite{alon2015online}. 

\begin{theorem}\label{thm.MAB_feedback}
	For stochastic multi-armed bandits with feedback graph $G$, with probability at least $1-\delta$, Algorithm \ref{algo.MAB} achieves regret upper bounds
	\begin{align*}
		R_T(\pi) &\le T \wedge  \begin{cases}
			50\alpha(\log K)^2 + 24\log K\cdot \sqrt{50\alpha T\log(2T)\log(KT/\delta)} &\text{\rm for Option I}, \\
			\alpha + 24\sqrt{\alpha T\log(2T)\log(KT/\delta)} &\text{\rm for Option II},
		\end{cases} \\
		&= \widetilde{O}(T \wedge \sqrt{\alpha T\log(1/\delta)}). 
	\end{align*}
\end{theorem}

Choosing $\delta = 1/T$, Theorem \ref{thm.MAB_feedback} gives an expected regret $\widetilde{O}(\sqrt{\alpha T})$ whenever $\alpha\le T$. To prove Theorem \ref{thm.MAB_feedback} we shall need the following concentration lemma. 
\begin{lemma}\label{lemma.concentration}
	With probability at least $1-\delta$, it holds that
	\begin{align*}
		\left|\bar{r}_a^t - R_a \right| \le 3\sqrt{\frac{\log(2T)\log(KT/\delta)}{n_t^a}}
	\end{align*}
	for every $t\in [T]$ and $a\in \calA$. 
\end{lemma}

With Lemma \ref{lemma.concentration} we are ready to prove Theorem \ref{thm.MAB_feedback}. 

\begin{proof}[Proof of Theorem \ref{thm.MAB_feedback}]
Let $\Delta_n := 3\sqrt{\log(2T)\log(KT/\delta)/n}$, and suppose the event in Lemma \ref{lemma.concentration} holds. During the period when $N=n\ge 1$, we have the following observations: 
\begin{itemize}
	\item The best arm $a^\star$ is not eliminated at the end of the period. This is because $\bar{r}_{a^\star}^t \ge R_{a^\star} - \Delta_n = \max_{a\in \calA_{\text{act}}} R_{a} - \Delta_n \ge \bar{r}_{\max}^t - 2\Delta_n$, where we note that $n_a^t \ge n$ for all $a\in \calA_{\text{act}}$ during this period. 
	\item All active arms $a\in \calA_{\text{act}}$ satisfy $R_a \ge R_{a^\star} - 4\Delta_n$. This is because $R_a \ge \bar{r}_a^t - \Delta_n \ge \bar{r}_{\max}^t - 3\Delta_n \ge \bar{r}_{a^\star}^t - 3\Delta_n \ge R_{a^\star} - 4\Delta_n$. 
\end{itemize}
Consequently, let $T_n$ be the duration of the period when $N=n$, with probability at least $1-\delta$ we have $R_T(\pi) \le T_0 + 4\sum_{n=1}^\infty T_n\Delta_n$. If we could prove that $T_n \le m$ for every $n\ge 0$, we would conclude that
\begin{align*}
	R_T(\pi) \le m + 4\sum_{n=1}^{\lceil T/m\rceil - 1} m\Delta_n \le m\left(1 + 24\sqrt{\log(2T)\log(KT/\delta)}\cdot \sqrt{\frac{T}{m}}\right). 
\end{align*}

Next we prove an upper bound of $T_n$ for each $n\ge 0$, which concludes the proof in view of the above arguments. 
\begin{itemize}
	\item Option I: consider the evolution of the set $\calA_0$ during the period when $N=n$. At each time, the vertex of the largest out-degree in $\calA_0$ is chosen, and this vertex with all its out-neighbors are removed from $\calA_0$. Using the inequality between the independence number and the weak domination number in Lemma \ref{lemma.independence_number}, and the fact that the independence number of $G_0$ is at most $\alpha$, the proportion of removed vertices from $\calA_0$ is at least $1/(50\alpha\log K)$. Therefore, the number of remaining vertices in $\calA_0$ after $m$ rounds is at most
	\begin{align*}
		K\cdot \left(1 - \frac{1}{50\alpha\log K}\right)^m < K\exp\left(-\frac{m}{50\alpha\log K}\right), 
	\end{align*}
	which is no larger than $1$ if $m\ge 50\alpha(\log K)^2$. Since the $n$-th period ends when $\calA_0$ becomes empty, we conclude that $T_n\le 50\alpha(\log K)^2$ must hold. 
	\item Option II: Let $a_1, \cdots, a_m$ be the chosen actions during the period when $N=n$, we prove that $\{a_1,\cdots,a_m\}$ must be an independence set of $G$. Consider any pair $i<j$. On one hand, there is no edge from $a_i$ to $a_j$, for $a_j$ would have been removed from $\calA_0$ when $a_i$ is chosen. On the other hand, there is no edge from $a_j$ to $a_i$, as by definition $a_i$ has no parent in $\calA_0$ when it is chosen. Consequently, there is no edge between $(a_i, a_j)$, and $m \le \alpha$.  
\end{itemize}
\end{proof}

\subsection{Stochastic Contextual Bandits with Graph Feedback and Cross Learning}\label{subsec.contextual_feedback}
In this section, we extend the framework of Section \ref{subsec.MAB_feedback} to contextual bandits, with the concept of cross learning introduced in \cite{balseiro2019contextual}. 

\begin{definition}[Stochastic Contextual Bandits with Graph Feedback and Cross Learning]\label{def.contextual_feedback_cross}
	A stochastic contextual bandit with graph feedback and cross learning consists of a time horizon $T$, a context space $\calC$ with $|\calC| = M$, an action space $\calA$ with $|\calA|=K$, a joint reward distribution $P$ supported on $[0,1]^{M\times K}$ with marginal distributions $(P_{c,a})_{c\in\calC, a\in\calA}$ and mean rewards $(R_{c,a})_{c\in\calC, a\in \calA}$, and a directed graph $G=(\calA, E)$ with vertex set $\calA$ and edge set $E$.
	\begin{enumerate}
		\item \text{\rm \bf Learning and Feedback Model.}
		At each $t\in [T]$, nature generates a context $c_t \in \calC$, either stochastically or adversarially, and reveals it to the learner. The learner then chooses an action $a_t\in \calA$ based on $c_t$ and the historical observations. Nature draws a reward vector $r_t = (r_{t,c,a})_{c\in \calC, a\in \calA}\overset{iid}{\sim} P$ and reveals $(r_{t,c,a})_{c\in \calC, a\in I_t}$ with $I_t = \{a_t\}\cup \{a\in\calA: (a_t,a)\in E\}$. 
		\item \text{\rm \bf Regret.} 
		Let $\pi = (a_1, a_2, \dots, a_T)$ be the overall policy used by the learner, then the regret is:
		\begin{align}\label{eq:contextual_bandit_regret}
			R_T(\pi) \triangleq \sum_{t=1}^T \left(\max_{a\in\calA} R_{c_t, a} - R_{c_t,a_t} \right). 
		\end{align}
		Note that the regret in \eqref{eq:contextual_bandit_regret} depends on the context sequence $(c_1,\cdots,c_T)$; this dependence is suppressed for notational simplicity. 
	\end{enumerate}
\end{definition}

Compared with Definition \ref{def.MAB_feedback}, Definition \ref{def.contextual_feedback_cross} combines the graph feedback across \emph{actions} and cross learning across \emph{contexts}. Specifically, whenever the learner chooses an action $a$ under context $c$, the rewards for all actions which are either $a$ or its out-neighbors in $G$ under \emph{all} contexts are revealed. In other words, the directed graph $G$ is the feedback structure across actions, and a complete graph is the feedback structure across contexts. In principle, as it was done in \cite{balseiro2019contextual,dann2020reinforcement}, one can also consider a general feedback graph across contexts, or assume different graphs $G_c$ under different contexts. However, we remark that this simplified model already captures non-trivial interplay between the contexts and actions, and covers the example of first-price auctions of interest. 

\begin{algorithm}[h!]
	\caption{Arm Elimination for Stochastic Contextual Bandits with Graph Feedback and Cross Learning}\label{algo.contextual}
	\textbf{Input:} Time horizon $T$; context set $\calC = [M]$; action set $\calA = [K]$; feedback graph $G$; failure probability $\delta\in (0,1)$. \\
	\textbf{Output:} A resulting policy $\pi$. \\
	\textbf{Initialization:}
	Active set $\calA_{\text{act},c} \gets \calA$, minimum count $N_c\gets 0$, for each $c\in \calC$;  
	
	All (empirical reward, action count) pairs $(\bar{r}_{c,a}^0, n_{a}^0)$ initialized to $(0,0)$ for each $c\in\calC, a\in \calA.$\\
	\For{$t\in\{1,2,\cdots,T\}$}{
		The learner receives the context $c_t \in \calC$; \\
		Let $\calA_{0, c_t} = \{a\in \calA_{\text{act}, c_t}: n_a^{t-1} = N_{c_t}\}$ be the set of active actions which have been pulled for the smallest amount of time under the context $c_t$, and $G_{0, c_t}$ be the restriction of $G$ to the vertex set $\calA_{0,c_t}$;\\
		The learner chooses the following action (break ties arbitrarily): \\
		\quad \textbf{(Option I, general graph)} $a_t \gets $ action in $\calA_{0,c_t}$ with the largest out-degree in $G_{0,c_t}$;\\
		\quad \textbf{(Option II, DAG)} $a_t \gets$ action in $\calA_{0,c_t}$ with in-degree zero in $G_{0,c_t}$;\\
		The learner receives random rewards $(r_{t,c,a})_{c\in\calC, a\in I_t}$ with $I_t = \{a_t\}\cup \{a\in \calA: (a_t, a)\in E \}$; \\
		The learner updates the empirical rewards and action counts $(\bar{r}_{c,a}^t, n_a^t)$ for $c\in\calC, a\in I_t$; \\
		\For{$c\in \calC$}{
			\If{$\min_{a\in \calA_{\text{\rm act},c}} n_a^t > N_c$}{
				Update the minimum count $N_c\gets N_c+1$; \\
				Let $\bar{r}_{c,\max}^t=\max_{a\in \calA_{\text{act},c}} \bar{r}_{c,a}^t$ be the maximum empirical reward under context $c$; \\
				Update the active set
				\begin{align*}
					\calA_{\text{act},c} \gets \left\{a\in \calA_{\text{act},c}: \bar{r}_{c,a}^t \ge \bar{r}_{c,\max}^t - 6\sqrt{\frac{\log(2T)\log(MKT/\delta)}{N_c}} \right\}. 
				\end{align*}
			}
		}
	}
\end{algorithm}

We propose a learning algorithm for contextual bandits in Algorithm \ref{algo.contextual}, which is a generalization of Algorithm \ref{algo.MAB} based on the idea of arm elimination. Algorithm \ref{algo.contextual} maintains an active set $\calA_c$ for each context $c\in \calC$, but thanks to the cross learning structure, the action count $n_{a}^t$ does not depend on $c$ and is the same across contexts. The rest of the algorithm is similar to Algorithm \ref{algo.MAB}: for each context $c$, whenever the minimum count $N_c=\min_{a\in \calA_c} n_a$ is increased by one, the active set $\calA_c$ is updated; between two consecutive updates of action sets, the ``most informative'' action is chosen according to either Option I or II. The following theorem summarizes the regret performance of Algorithm \ref{algo.contextual}. 

\begin{theorem}\label{thm.contextual}
	For stochastic contextual bandits with feedback graph $G$ and cross learning (the contexts could either be stochastic or adversarial), with probability at least $1-\delta$, Algorithm \ref{algo.contextual} achieves regret upper bounds
	\begin{align*}
		R_T(\pi) &\le T \wedge \left(K + 3\sqrt{KT\log(2T)\log(MKT/\delta)}\right) \\
		&\qquad \wedge\begin{cases}
			50\alpha M(\log K)^2 + 24\log K\cdot \sqrt{50\alpha M T\log(2T)\log(MKT/\delta)} &\text{\rm for Option I}, \\
			\alpha M + 24\sqrt{\alpha M T\log(2T)\log(MKT/\delta)} &\text{\rm for Option II}, 
		\end{cases} \\
		&= \widetilde{O}(T \wedge \sqrt{\alpha M T\log(1/\delta)} \wedge \sqrt{KT\log(1/\delta)}). 
	\end{align*}
\end{theorem}

Theorem \ref{thm.contextual} shows that Algorithm \ref{algo.contextual} achieves ``best of both worlds'', i.e. a single algorithm achieves the best performances simultaneously when the learner completely ignores the cross learning or feedback graph structure. However, whenever $\min\{\alpha M, K\}\gg 1$, the regret upper bound in Theorem \ref{thm.contextual} is still much larger than $\widetilde{O}(\sqrt{T})$, a pessimistic result in view of the rich feedback structures. This point will be elaborated in later sections. 

\begin{proof}[Proof of Theorem \ref{thm.contextual}]
First note that by applying a union bound over $t\in [T], c\in \calC$ and $a\in \calA$, a similar result of Lemma \ref{lemma.concentration} (with $\delta$ replaced by $\delta/M$) holds in the contextual case for all $c\in \calC$.

We now prove the second upper bound of $R_T(\pi)$. Let $T_c = \sum_{t=1}^T \1(c_t=c)$ be the number of occurrences of context $c$, we may decompose the regret $R_T(\pi)$ in \eqref{eq:contextual_bandit_regret} into the sum of regrets under different contexts $c$, with different time horizons $T_c$. For the $c$-th subproblem, similar arguments of the proof of Theorem \ref{thm.MAB_feedback} still work through, and we only need to additionally note that the actions under other contexts $c'$ only make the duration of the period when $N_c = n$ in the $c$-th subproblem no longer (as $\calA_{0,c}$ never expands given a fixed $N_c$). Consequently, summing up the regret upper bound of Theorem \ref{thm.MAB_feedback} for different contexts $c$, and using $\sum_{c\in \calC} \sqrt{T_c} \le \sqrt{MT}$ completes the proof. 

Next we prove the first upper bound of $R_T(\pi)$. To this end, we sum over different actions. For action $a\in \calA$, let $t_1<t_2<\cdots<t_{T_a}$ be the time points when $a$ is chosen. If $a$ is chosen at time $t_m$ with context $c_m$, in both Option I and II, $n_{a}^{t_m - 1}$ must be the smallest value in $\{n_{a'}^{t_m-1}: a'\in \calA_{\text{act}, c_m} \}$. Using the second observation in the proof of Theorem \ref{thm.MAB_feedback}, the suboptimality of choosing action $a$ is at most $4\Delta(n_a^{t_m-1})$, where $\Delta(n) := 3\sqrt{\log(2T)\log(MKT/\delta)/n}$. Moreover, thanks to the cross learning structure, whenever $a$ is chosen regardless of the context, the quantity $n_a$ is increased by one. Consequently, $n_a^{t_m-1} \ge m-1$, and the total regret incurred by choosing action $a$ is at most
\begin{align*}
	1 + \Delta(1) + \cdots + \Delta(T_a - 1) \le 1 + 6\sqrt{\log(2T)\log(MKT/\delta)T_a}. 
\end{align*}
Summing over $a$ and using $\sum_{a\in \calA}T_a = T$ completes the proof of the first upper bound. 
\end{proof}

\begin{remark}
	Using a UCB-type algorithm, it was shown in \cite{dann2020reinforcement} that the $\widetilde{O}(\sqrt{KT})$ upper bound could be improved to $\widetilde{O}(\sqrt{mT})$, where $m\le K$ is the size of the maximum acyclic subgraph of $G$. However, the particular focus of the coming sections is the case where $G$ is a DAG, so $m=K$ and this improvement is not substantial. 
\end{remark}

\subsection{Partially Ordered Contextual Bandits}\label{subsec.partial_order}
A natural question from Theorem \ref{thm.contextual} is whether or not the learner could jointly exploit the graph feedback and cross learning structures, and importantly, eliminate the polynomial dependence on $\min\{K, M\}$ in the final regret. In this section we present such a partial ordering assumption that a small regret is possible. 

To this end we introduce the concept of \emph{partially ordered contextual bandits}: 
\begin{definition}[Partially Ordered Contextual Bandits]\label{assump:partial_order} We call a bandit instance in Definition \ref{def.contextual_feedback_cross} \emph{partially ordered contextual bandit} if the following conditions additionally hold: 
	\begin{itemize}
		\item the feedback graph $G$ is a DAG whose transitive closure is itself; in other words, every descedent of a node $v$ in $G$ is a child of $v$. This defines a partial order $(\calA, \preceq_G)$ on the action set $\calA$: $a_1 \preceq_G a_2$ if and only if either $a_1 = a_2$, or $a_2$ is a child of $a_1$ in $G$; 
		\item there is a partial order $(\calC, \preceq_\calC)$ over $\calC$ such that if $c_1\preceq_\calC c_2$, then $a^\star(c_1) \preceq_G a^\star(c_2)$, where $a^\star(c):=\arg\max_{a\in \calA} R_{c,a}$ is the optimal action under context $c$. 
	\end{itemize}
\end{definition} 
Definition \ref{assump:partial_order} imposes an ordering on both the optimal actions and the contexts: when the context moves from small to large (in terms of the partial order $\preceq_\calC$), the optimal action under the context also moves from small to large (in terms of the partial order $\preceq_G$). If $(\calC, \preceq_\calC)$ is the empty relation, Definition \ref{assump:partial_order} is vacuous, and this is essentially the setting in Section 3.2 without further assumptions; for first-price auctions, Section \ref{subsec.application_auction} will show that both $(\calA, \preceq_G)$ and $(\calC, \preceq_\calC)$ are chains (i.e. totally ordered sets - any pair is comparable). 

\begin{algorithm}[h!]
	\caption{Arm Elimination for Partially Ordered Contextual Bandits}\label{algo.partial_order}
	\textbf{Input:} Time horizon $T$; context set $\calC = [M]$; action set $\calA = [K]$; feedback graph $G$; partial orders $(\calA, \preceq_G)$ and $(\calC, \preceq_\calC)$; failure probability $\delta\in (0,1)$. \\
	\textbf{Output:} A resulting policy $\pi$. \\
	\textbf{Initialization:}
	Active set $\calA_{\text{act},c} \gets \calA$, minimum count $N_c\gets 0$, for each $c\in \calC$;  
	
	All (empirical reward, action count) pairs $(\bar{r}_{c,a}^0, n_{a}^0)$ initialized to $(0,0)$ for each $c\in\calC, a\in \calA.$\\
	\For{$t\in\{1,2,\cdots,T\}$}{
		The learner receives the context $c_t \in \calC$; \\
		Let $G_{\text{act},c_t}$ be the restriction of $G$ to the vertex set $\calA_{\text{act}, c_t}$; \\
		($\star$) The learner picks $a_t \gets $ a uniformly random action in $\calA_{\text{act},c_t}$ with no parent in $G_{\text{act},c_t}$; \\
		The learner receives random rewards $(r_{t,c,a})_{c\in\calC, a\in I_t}$ with $I_t = \{a_t\}\cup \{a\in \calA: (a_t, a)\in E \}$; \\
		The learner updates the empirical rewards and action counts $(\bar{r}_{c,a}^t, n_a^t)$ for $c\in\calC, a\in I_t$; \\
		\For{$c$ in a given topological ordering of $\calC$}{
			($\star\star$) Eliminate all actions $a$ from $\calA_{\text{act},c}$ if there exists $c'\preceq_\calC c$ such that no action $a'\in \calA_{\text{act}, c'}$ satisfies $a'\preceq_G a$; \\
			\If{$\min_{a\in \calA_{\text{\rm act},c}} n_a^t > N_c$}{
				Update the minimum count $N_c\gets \min_{a\in \calA_{\text{\rm act},c}} n_a^t$; \\
				Let $\bar{r}_{c,\max}^t=\max_{a\in \calA_{\text{act},c}} \bar{r}_{c,a}^t$ be the maximum empirical reward under context $c$; \\
				Update the active set
				\begin{align*}
					\calA_{\text{act},c} \gets \left\{a\in \calA_{\text{act},c}: \bar{r}_{c,a}^t \ge \bar{r}_{c,\max}^t - 6\sqrt{\frac{\log(2T)\log(MKT/\delta)}{N_c}} \right\}. 
				\end{align*}
			}
		}
	}
\end{algorithm}

For partially ordered contextual bandits, one can slightly modify Algorithm \ref{algo.contextual} to eliminate more actions in the update of active action sets, shown in Algorithm \ref{algo.partial_order}. There are two main differences bewteen Algorithms \ref{algo.contextual} and \ref{algo.partial_order}. First, instead of choosing an action which has been pulled for the least amount of time in the active set in Algorithm \ref{algo.contextual}, Line ($\star$) of Algorithm \ref{algo.partial_order} chooses the action directly from the active set regardless of how many times it has been pulled. As a comparison, the rule in Algorithm \ref{algo.contextual} is deterministic and ensures that for two consecutive updates of the active set $\calA_{\text{act}, c}$ at time $t_1, t_2$, the duration must satisfy $\sum_{t=t_1+1}^{t_2} \1(c_t = c) \le \alpha$; in contrast, the rule in Algorithm \ref{algo.partial_order} is \emph{random} and ensures that the above duration is $\widetilde{O}(\alpha)$ with high probability. The main benefit of the random rule is that, the definition of duration changes and a stronger inequality $\sum_{t=t_1+1}^{t_2} \1(c_t \preceq_\calC c) = \widetilde{O}(\alpha)$ holds for partially ordered contextual bandits (cf. Lemma \ref{lemma.partial_order_N_lb}). This means that choosing an action under a small context $c_1$ provides useful information for $\widetilde{\Omega}(1/\alpha)$ fraction of the active actions under a large context $c_2$, a key property to arrive at a small regret. 

The second difference is the introduction of an additional arm elimination rule (Line ($\star\star$)) thanks to the partial order: for $c_1\preceq_\calC c_2$, any action $a\in \calA_{\text{act}, c_2}$ must be dominated by an action in $\calA_{\text{act},c_1}$. This is because the active set $\calA_{\text{act},c}$ is constructed so that $a^\star(c)\in \calA_{\text{act},c}$ with high probability; therefore, if an action $a$ is not dominated by any action in $\calA_{\text{act},c}$, Definition \ref{assump:partial_order} ensures that $a\neq a^\star(c')$ for every $c'\succeq_\calC c$, and we can eliminate $a$ from $\calA_{\text{act}, c'}$. 

The main benefit of Definition \ref{assump:partial_order} is that, when the contexts are i.i.d., a smaller regret depending on the independence number $\beta$ of $(\calC, \preceq_\calC)$ is achievable. Here the independence number of a partial order $(X, \preceq)$ is the largest cardinality of $\{x_1,\cdots,x_m\}\subseteq X$ such that any pair is incomparable. The regret performance of Algorithm \ref{algo.partial_order} is illustrated in the following theorem. 

\begin{theorem}\label{thm:partial_order}
	Consider any partially ordered contextual bandit in Definition \ref{assump:partial_order}, with $\alpha$ and $\beta$ denoting the independence number of $G$ and $(\calC, \preceq_\calC)$, respectively. If the contexts $(c_t)_{t \in [T]}$ follows an i.i.d. distribution $F$, then choosing $\delta = 1/(2T)$, the policy in Algorithm \ref{algo.partial_order} satisfies
	\begin{align*}
		\bE[R_T(\pi)] \le 1 + 24\sqrt{\alpha \beta T\log(2T)\log(2MKT^2)\log(2MKT^3)}(1+\log T) = \widetilde{O}(\sqrt{\alpha\beta T}), 
	\end{align*}
	where the expectation is taken with respect to the random contexts. Moreover, if the contexts are adversarially generated, with probability at least $1-2\delta$ the same algorithm achieves
	\begin{align*}
		R_T(\pi) \le 24\sqrt{\alpha MT\log(2T)\log(MKT/\delta)\log(MKT^2/\delta)(1+\log T)} = \widetilde{O}(\sqrt{\alpha MT}\log(1/\delta)). 
	\end{align*}
\end{theorem}

Theorem \ref{thm:partial_order} shows that, for partially ordered contextual bandits, Algorithm \ref{algo.partial_order} achieves an expected $\widetilde{O}(\sqrt{\alpha\beta T})$ regret under stochastic contexts which deteriorates to $\widetilde{O}(\sqrt{\alpha MT})$ regret under adversarial contexts. Consequently, when both quantities $\alpha$ and $\beta$ are of the order $\widetilde{O}(1)$, a small regret $\widetilde{O}(\sqrt{T})$ independent of $(K, M)$ is achievable in expectation under Definition \ref{assump:partial_order} and \emph{stochastic contexts}. If the contexts are adversarial, we know that Algorithm \ref{algo.contextual} achieves an $\widetilde{O}(\min\{\sqrt{\alpha MT}, \sqrt{KT}\})$ regret (cf. Theorem \ref{thm.contextual}), and Theorem \ref{thm:partial_order} shows that Algorithm \ref{algo.partial_order} partially attains it and achieves $\widetilde{O}(\sqrt{\alpha MT})$ regret. We do not know whether a ``best of both worlds'' result could be obtained for Algorithm \ref{algo.partial_order} or others. 

\begin{proof}[Proof of Theorem \ref{thm:partial_order}]
We first fix the possibly random contexts $(c_t)_{t\in [T]}$ and present a generic analysis that works under both stochastic and adversarial contexts. By similar lines to the proof of Theorem \ref{thm.MAB_feedback}, conditioned on an event $\calE_1$ which happens with probability at least $1-\delta$, if an action $a\in \calA$ is chosen at time $t$, the suboptimality gap of $a$ under the corresponding context $c_t$ is at most $4\Delta(N_{c_t}^{t-1})$, where $\Delta(n):= 3\sqrt{\log(2T)\log(MKT/\delta)/n}$. Consequently,
\begin{align*}
	R_T(\pi)\1(\calE_1) \le \sum_{t=1}^T \min\left\{4\Delta(N_{c_t}^{t-1}), 1 \right\}, 
\end{align*}
The main benefit of Definition \ref{assump:partial_order} lies in the following lower bound on $N_{c_t}^{t-1}$. 
\begin{lemma}\label{lemma.partial_order_N_lb}
	With probability at least $1-\delta$, for every $t\in [T]$ it holds that
	\begin{align*}
		N_{c_t}^{t-1} \ge \frac{\sum_{s<t} \1(c_s \preceq_\calC c_t)}{\alpha\log(MKT^2/\delta)} - 1. 
	\end{align*} 
\end{lemma}

Lemma \ref{lemma.partial_order_N_lb} formalizes the following idea: in Algorithm \ref{algo.partial_order}, choosing a small context provides a fresh observation under a large context. Consequently, the quantity $N_c^t$ could stabilize during $[t_1, t_2]$ only if the ``effective sample size'' for context $c$ satisfies $\sum_{t=t_1+1}^{t_2} \1(c_t\preceq_\calC c) = \widetilde{O}(\alpha)$. This is the main benefit of the partial order structure in Definition \ref{assump:partial_order}. 

To proceed, let $\calE_2$ be the good event ensured by Lemma \ref{lemma.partial_order_N_lb}, then
\begin{align*}
	R_T(\pi)\1(\calE_1 \cap \calE_2) &\le \sum_{t=1}^T \min\left\{4\Delta(N_{c_t}^{t-1}), 1 \right\}\1(\calE_2) \le \sum_{t=1}^T 8\Delta(N_{c_t}^{t-1}+2)\1(\calE_2) \\
	&\le 24\sqrt{\alpha\log(2T)\log(MKT/\delta)\log(MKT^2/\delta)}\sum_{t=1}^T \frac{1}{\sqrt{1+\sum_{s<t} \1(c_s\preceq_\calC c_t) }} \\
	&\le 24\sqrt{\alpha T\log(2T)\log(MKT/\delta)\log(MKT^2/\delta)}\cdot \sqrt{\sum_{t=1}^T \frac{1}{1+\sum_{s<t}\1(c_s\preceq_\calC c_t)}}.
\end{align*}
The next lemma presents two upper bounds on the final sum, depending on whether the contexts are stochastic or adversarial. 

\begin{lemma}\label{lemma.record_breaking}
	\begin{enumerate}
		\item 
		Let $X_1,\cdots,X_T$ be iid real-valued random variables taking value in a finite partially ordered set $(\calX, \preceq)$, which has independence number $\beta$. Then
		\begin{align*}
			\bE\left[ \sum_{t=1}^T \left(1 + \sum_{s<t} \1(X_s \preceq X_t) \right)^{-1} \right] \le \beta(1+\log T)^2. 
		\end{align*}
		\item 
		Let $x_1, \dots, x_T$ be arbitrary (possibly generated adversarially) in a partially ordered set $(\calX, \preceq)$. Then
		\begin{align*}
			\sum_{t=1}^T \left(1 + \sum_{s<t} \1(x_s \preceq x_t) \right)^{-1} \le |\calX| (1+\log T). 
		\end{align*}
	\end{enumerate}
\end{lemma}

Lemma \ref{lemma.record_breaking} is the key place where stochastic contexts offer greater help than adversarial contexts: the upper bound improves from $|\calC| = M$ to $\beta$. Plugging different results of Lemma \ref{lemma.record_breaking} into the regret upper bound leads to the claimed results of Theorem \ref{thm:partial_order}. 
\end{proof} 

\subsection{Application to Repeated First-price Auctions}\label{subsec.application_auction}
We now describe how learning in first-price acutions is a special instance of the partially ordered contextual bandit in Definition \ref{assump:partial_order}. First, let $c_t = v_t$ (private values are contexts) and $a_t = b_t$ (bids are actions), then the joint distribution $P$ over the rewards $(r_{t,c,a})_{c\in\calC, a\in\calA}$ is given by:
\begin{align*}
	r_{t,c,a} = (c - a)\1(m_t \le a),\qquad \forall c\in\calC, a\in\calA, 
\end{align*}
where $m_t \overset{iid}{\sim} G$ is a common random variable shared among all contexts and actions. 
Further, the mean rewards are $R_{c,a} = (c-a)G(a)$.
The only remaining issue is that the private values and bids come from the continuous space $[0,1]$,
which we resolve by the following simple quantization: 
\begin{align}\label{eq:quantization}
	\calC := \left\{\frac{1}{M}, \frac{2}{M}, \cdots, 1\right\}, \qquad \calA := \left\{\frac{1}{K}, \frac{2}{K}, \cdots, 1\right\}. 
\end{align}
Note that any policy $\pi^{\text{\rm Q}}$ on this quantized problem naturally results in a policy in the original (continuous-space) problem: when $v_t \in [0,1]$ appears, first clip it to $\widetilde{v}_t = \min\{v\in\calC: v\ge v_t \}$ and apply $\pi^{\text{\rm Q}}$ to get the bid for time $t$. The following lemma shows that this quantization causes little performance degradation.

\begin{lemma}\label{lemma.approx_regret}
	Let $\pi^{\text{\rm Q}}$ be the policy for the quantized problem and $\pi$ be the induced policy for the original problem. Let $\widetilde{R}_T(\pi^{\text{\rm Q}},v) = \sum_{t=1}^T \left(\max_{b\in \calA} R(\widetilde{v}_t, b) - R(\widetilde{v}_t,b_t) \right)$ be the regret of $\pi^{\text{\rm Q}}$ in the quantized problem, where 
	$v = (v_1, \dots, v_T)$ is an arbitrary sequence of private values in the original problem.
	Then: $$R_T(\pi,v) \le \widetilde{R}_T(\pi^{\text{\rm Q}},v) + \left(\frac{2}{M} + \frac{1}{K}\right)T.$$
\end{lemma}

We are left with checking the graph feedback and cross learning structures, as well as the partial orders in Definition \ref{assump:partial_order}. We claim that the feedback graph $G$ is a chain: for two bids $b, b' \in \calA$, there is an edge from $b$ to $b'$ if and only if $b<b'$. Clearly $G$ satisfies the first condition of Definition \ref{assump:partial_order}, with independence number $\alpha=1$. To see why $G$ is a valid feedback structure in first-price auctions, recall that the bidder receives partial information $\1(b_t\ge m_t)$ and $m_t\1(b_t<m_t)$ at the end of time $t$. As such, all indicator functions $\1(b\ge m_t)$ are also observable for any $b\ge b_t$, and so are the instantaneous rewards $r(v,b;m_t)=(v-b)\1(b\ge m_t)$ for all $v\in [0,1]$ and $b\ge b_t$. In other words, a bid $b_t$ not only reveals the reward for this particular bid under $v_t$, but also reveals the rewards for all bids $b\ge b_t$ and all values $v$ - this is the desired structure of graph feedback across actions and cross learning across contexts. 

Finally, to verify the second condition of Definition \ref{assump:partial_order}, we just take $\calC$ as a totally ordered set with $v_1 \preceq_\calC v_2 \Leftrightarrow v_1 \le v_2$. Consequently, the second condition of Definition \ref{assump:partial_order} is validated by the following lemma. 

\begin{lemma}\label{lemma.monotonicity}
	For each $v\in [0,1]$, let $b^\star(v) = \arg\max_{b\in \calA} R(v,b)$ (if there are multiple maximizers in the finite set $\calA$, take the largest). Then the mapping $v\mapsto b^\star(v)$ is non-decreasing. 
\end{lemma}

Lemma \ref{lemma.monotonicity} simply states that, whenever the bidder has a higher valuation of an item, she is willing to bid a higher price for it: although $b^\star(v)$ are unknown in general, a reliable lower bound of $b^\star(v)$ translates to a reliable lower bound of $b^\star(v')$ for all $v'\ge v$.

Based on the above properties, we see that learning in first-price auctions fit into the framework of partially ordered contextual bandits, with $\alpha=\beta=1$. The following corollary is then immediate from Theorem \ref{thm:partial_order}: 

\begin{corollary}\label{cor.stochastic}
	Algorithm \ref{algo.partial_order} applied to the quantized first-price auction leads to an expected regret $O(\sqrt{T}\log^{2.5}T)$ when the private values are i.i.d. (with parameters $M=K=T^{1/2}$ and $\delta = 1/(2T)$), and an expected regret $O(T^{2/3}\log^2 T)$ when the private values are adversarially generated (with parameters $M=K=T^{1/3}$ and $\delta = 1/(2T)$). 
\end{corollary}

Corollary \ref{cor.stochastic} shows that, applying our general framework of contextual bandits with both graph feedback and cross learning is sufficient to achieve an $\widetilde{O}(\sqrt{T})$ regret for first-price auctions, when the private values are stochastic. When the private values are adversarial, the next section will that the partially ordered contextual bandit framework turns out to be insufficient.

\subsection{Lower Bounds under Adversarial Contexts}\label{subsec.lowerbound}
The previous sections show that for partially ordered contextual bandits, a small regret $\widetilde{O}(\sqrt{\alpha\beta T})$ almost independent of $(K, M)$ is attainable in expectation under stochastic contexts. A natural question is whether similar guarantees are also possible under adversarial contexts. This section proves that this is not possible in general: there is a contextual bandit instance satisfying Definition \ref{assump:partial_order} with $\beta = 1$, while any learning algorithm suffers from a regret $\Omega(\min\{\sqrt{\alpha MT}, \sqrt{KT}, T^{2/3}\})$ in the worst case. 
The formal statement is summarized in the following theorem. 

\begin{theorem}\label{thm.lower_bound}
	For every $\alpha\in \mathbb{N}$, there exists a graph $G$ with independence number $\alpha$ such that
	\begin{align*}
		\inf_\pi \sup_{P}\sup_{\{c_t\}_{t=1}^T} \bE[R_T(\pi)] \ge c \min\left\{\sqrt{\alpha MT}, \sqrt{KT}, T^{2/3} \right\}.
	\end{align*}
	Here the supremum is taken over all possible partially ordered contextual bandits with feedback graph $G$, joint reward distribution $P$ on $[0,1]^{\calC\times \calA}$ such that the context set $\calC$ is totally ordered (i.e. $\beta=1$), and all possible oblivious context sequences $(c_t)_{t\in [T]}$; the infimum is taken over all possible policies, and $c>0$ is an absolute constant independent of $(K,M,T,\alpha)$. 
\end{theorem}

Note that the minimum of first two terms $\Omega(\min\{\sqrt{\alpha MT}, \sqrt{KT}\} )$ matches the upper bound in Theorem \ref{thm.contextual} within logarithmic factors, showing that under adversarial contexts, using an algorithm without using Definition \ref{assump:partial_order} is somewhat optimal. We remark that the last term $T^{2/3}$ in the lower bound is not superfluous as well, for there exists another explore-then-commit (ETC) policy that achieves an $\widetilde{O}(T^{2/3})$ regret regardless of $(M, K)$. Consider the case $\alpha = 1$, meaning that there is an action $a_0$ (in the DAG case; for the general case, $a_0$ is a union of $O(\log K)$ actions in view of Lemma \ref{lemma.independence_number}) which could provide full reward information for all actions and contexts $(r_{t,c,a})_{c\in \calC, a\in \calA}$. The ETC policy is then as follows: choose the action $a_0$ for the first $O(T^{2/3})$ rounds (exploration phase), and then greedily choose the action with the highest empirical reward for the rest of the time (commit phase). This policy leads to an $\widetilde{O}(T^{2/3})$ regret as the uniform estimation error after the exploration phase is $\widetilde{O}(T^{-1/3})$, giving an $\widetilde{O}(T^{2/3})$ regret in the commit phase. Combined with Theorem \ref{thm.contextual}, this actually results in a tight regret characterization for contextual bandits in Definition \ref{def.contextual_feedback_cross} when $\alpha = 1$: 

\begin{corollary}\label{cor.minimax_regret}
	For any given $(M, K, T)$, the minimax regret, where the maximum takes over all possible bandit instances in Definition \ref{def.contextual_feedback_cross} with $\alpha=1$ (regardless of whether or not Definition \ref{assump:partial_order} is additionally required), is $\widetilde{\Theta}(\sqrt{\min\{K, M, T^{1/3}\}T})$. 
\end{corollary}

We use a Bayesian approach to prove Theorem \ref{thm.lower_bound}: we set $G$ as a disjoint union of $\alpha$ chains, and construct $(\alpha+1)^M$ reward distributions $P^u$ indexed by $u \in \{0,1,\cdots,\alpha\}^M$. We show that under the uniform mixture of these reward distributions, no single policy can incur an average regret smaller than $\Omega(\sqrt{\min\{K, \alpha M, T^{1/3}\} T})$. The derivation of the above average lower bound requires a delicate exploration-exploitation tradeoff: we show that if the average regret were small, then the learner would not have enough information to distinguish among the components of the mixture and would in turn incur a large average regret. Note that this is different from typical lower bounds in the bandits literature, where the inability of distinguishing among the components (or multiple hypotheses) for the learner is independent of the final regret she attains. 

\begin{proof}[Proof of Theorem \ref{thm.lower_bound}]
We first claim that it suffices to prove a regret lower bound of $\Omega(\sqrt{KT})$ under the scaling $K = (\alpha+1)M\le T^{1/3}$. To see why this is sufficient, for general $(M, K)$ we define $K' = \min\{K, (\alpha+1)M, T^{1/3}\}\le K$ and $M' = \min\{K, (\alpha+1)M, T^{1/3}\}/(\alpha+1)\le M$ (without loss of generality assume that both $K'$ and $M'$ are integers). By adding dummy contexts and actions, the minimax regret under $(M, K)$ is no smaller than that under $(M', K')$, where $K' = (\alpha+1)M' \le T^{1/3}$ holds. Using the minimax lower bound under this special scaling, the general lower bound is then:
\begin{align*}
	\Omega\left(\sqrt{K'T} \right) = \Omega\left(\min\left\{\sqrt{KT}, \sqrt{\alpha MT}, T^{2/3} \right\}\right).
\end{align*}

Under the scaling $K = (\alpha+1)M\le T^{1/3}$, we define the action space $\calA$ and feedback graph $G$ as follows. The action space is represented by $\calA = \{(i, j): i \in [M], j\in [\alpha]\cup \{0\} \}$, and the edges of the feedback graph $G$ are all $(i_1,j_1)\to (i_2,j_2)$ with $i_1<i_2$ and $(i,j)\to (i,0)$ with $j\neq 0$. Clearly $G$ is a DAG satisfying the first condition of Definition \ref{assump:partial_order}, and the independence number of $G$ is $\alpha$. 

Next we define a class of reward distributions $(P^u)_{u}$. The context space is simply $\calC = [M]$, and for each $u=(u_1,\cdots,u_M)\in \{0,1,\cdots,\alpha\}^M$, define 
\begin{align*}
	P^u = \prod_{c\in \calC, a\in \calA} P_{c,a}^u = \prod_{c\in [M]}\prod_{(i,j)\in [M]\times \{0,1,\cdots,\alpha\}} \mathsf{Bern}(R_{c,(i,j)}^u), 
\end{align*}
where the rewards $(R_{c,(i,j)}^u)$ are chosen as follows: 
\begin{align*}
	R_{c,(i,j)}^u = \begin{cases}
		\frac{1}{4} + \Delta &\text{if } i = c \text{ and } j = 0, \\
		\frac{1}{4} + 2\Delta &\text{if } i = c \text{ and } j \neq 0 \text{ and } u_c = j, \\
		\frac{1}{4} &\text{if } i = c \text{ and } j \neq 0\text{ and } u_c \neq j, \\
		0 &\text{if } i \neq c,
	\end{cases} \quad \text{ where }\Delta := \sqrt{\frac{K}{16T}} \le \frac{1}{4}. 
\end{align*}
The above rewards are chosen so that the following properties hold: 
\begin{enumerate}
	\item all mean rewards lie in $[0,1]$, and all non-zero mean rewards lie in $[1/4, 3/4]$; 
	\item under true parameter $u$ and context $c$, the unique optimal action is $(c, u_c)$; 
	\item for every $u$ and $c$, the instantaneous regret incurred by choosing a non-optimal action is at least $\Delta$; 
	\item for every context $c$, the instantaneous regret incurred by choosing any action $(i,j)$ with $i\neq c$ is at least $1/4$; 
	\item for $u=(u_1,\cdots,u_M)$ and $u'=(u_1,\cdots,u_{c-1},0,u_{c+1},\cdots,u_M)$ only differing in the $c$-th component, the KL divergence between the observed reward distributions $P^u(i,j)$ and $P^{u'}(i,j)$ when the action $(i,j)$ is chosen (under any context) is: 
	\begin{align*}
		D_{\text{KL}}(P^{u'}(i,j) \| P^u(i,j)) &= \begin{cases}
			0 &\text{if } i > c \text{ or } (i=c, j\neq u_c), \\
			D_{\text{KL}}(\mathsf{Bern}(1/4)\|\mathsf{Bern}(1/4+2\Delta)) & \text{if } i<c \text{ or }(i=c, j=u_c), 
		\end{cases} \\
		&\le \frac{64}{3}\Delta^2\cdot \1(i<c \text{ or } (i,j) = (c,u_c)), 
	\end{align*}
	where the second inequality is due to $D_{\text{KL}}(\mathsf{Bern}(p) \| \mathsf{Bern}(q)) \le (p-q)^2/(q(1-q))$. This is because $R^u$ and $R^{u'}$ only differ in the index $(c, (c, u_c))$, and we have used the feedback structure of $G$ and the cross learning across contexts. 
\end{enumerate}

Next we check the second condition of Definition \ref{assump:partial_order} with $\beta = 1$. The partial ordering on $\calC=[M]$ is simply the usual integer order: $c_1 \preceq_\calC c_2$ if and only if $c_1 \le c_2$. By the second property above and our construction of $G$, under any $u$ we have $a_u^\star(c_1)\preceq_G a_u^\star(c_2)$ as long as $c_1\le c_2$, where $a_u^\star(c)$ is the best action under $P^u$ and context $c$. Also it is clear that $\beta = 1$ for this total ordering. 

Finally we set the context sequence $(c_t)_{t=1}^T$. Since $M\le T^{1/3}$, without loss of generality we assume that $T/M$ is an integer. We split the time horizon $[T]$ into $M$ consecutive blocks, each of size $T/M$. During the $m$-th block, the context is always $c_m := M+1-m$. The reason why we choose this sequence is that the optimal action gets smaller (in terms of $\preceq_G$) over time, so early actions do not provide sufficient information for late actions, which makes learning harder. 

So far we have finished all preparations for the lower bound proof. Let $\pi$ be an arbitrary policy, and write the $t$-th action as $a_t = (i_t, j_t)\in \calA$. Let $R_T$ be the largest expected regret incurred by policy $\pi$ under true reward distributions $P^u$ when $u$ traverses $\{0,1,\cdots,\alpha\}^M$, our target is to show that $R_T = \Omega(\sqrt{KT})$. We shall prove two different lower bounds on $R_T$ and combine them to arrive at the desired result. First, by the fourth property above, for every $u=(u_1,\cdots,u_M)$, 
\begin{align}\label{eq:lowerbound_first}
	R_T &\ge \frac{1}{4}\sum_{m=1}^M\sum_{t\in T_m}\bE_{(P^{u})^{\otimes (t-1)}}[\1(i_t \neq c_m)], 
\end{align}
where $T_m$ is the $m$-th block, and $P^{\otimes t}$ denotes the distribution of all observables up to time $t$ under reward distribution $P$. 

The second lower bound  is more delicate. Let $u^m:=(u_1,\cdots,u_{m-1},0,u_{m+1},\cdots,u_M)$ set the $m$-th component of $u$ to zero. If $u$ is uniformly distributed over $\{0,1,\cdots,\alpha\}^M$, the Bayes regret is at least
\begin{align*}
	R_T &\stepa{\ge} \bE_u\left[ \sum_{m=1}^M \sum_{t\in T_m} \Delta\cdot \bE_{(P^u)^{\otimes (t-1)}}[\1(j_t \neq u_{c_m})] \right] \\
	&\stepb{=} \Delta\cdot \sum_{m=1}^M\sum_{t\in T_m} \bE_{u\backslash \{u_{c_m} \}}\left\{ \bE_{u_{c_m}}[ \bE_{(P^u)^{\otimes (t-1)}}[\1(j_t \neq u_{c_m})] ] \right\} \\
	&\stepc{\ge} \Delta\cdot \sum_{m=1}^M\sum_{t\in T_m} \bE_{u\backslash \{u_{c_m} \}}\left\{ \frac{\alpha}{2(\alpha+1)}\exp\left( - \frac{1}{\alpha}\sum_{u_{c_m}=1}^\alpha D_{\text{KL}}( (P^{u^{c_m}})^{\otimes (t-1)}  \| (P^{u})^{\otimes (t-1)}) \right) \right\} \\
	&\stepd{\ge} \frac{\Delta}{4} \cdot \sum_{m=1}^M \sum_{t\in T_m} \bE_{u\backslash \{u_{c_m} \}}\exp\left( -\frac{64\Delta^2}{3\alpha}\sum_{u_{c_m}=1}^\alpha \sum_{s<t} \bE_{ (P^{u^{c_m}})^{\otimes (s-1)}}[\1(i_s < c_m \text{ or } (i_s, j_s) = (c_m, u_{c_m}))] \right),
\end{align*}
where (a) is due to the third property of mean rewards, (b) decomposes the uniformly distributed $u$ into uniformly distributed $u\backslash \{u_{c_m}\}$ and $u_{c_m}$, (c) is due to Lemma \ref{lemma.tree} and the convexity of $x\mapsto e^{-x}$, and (d) follows from the chain rule of KL divergence and the fifth property of mean rewards. To deal with the remaining exponent, we have
\begin{align*}
	&\sum_{s<t} \bE_{ (P^{u^{c_m}})^{\otimes (s-1)}}[\1(i_s < c_m \text{ or } (i_s, j_s) = (c_m, u_{c_m}))] \\
	&= \sum_{m'<m}^{}\sum_{s\in T_{m'}}  \bE_{ (P^{u^{c_m}})^{\otimes (s-1)}}[\1(i_s < c_m \text{ or } (i_s, j_s) = (c_m, u_{c_m}))] \\
	&\qquad + \sum_{s<t, s\in T_m} \bE_{ (P^{u^{c_m}})^{\otimes (s-1)}}[\1(i_s < c_m \text{ or } (i_s, j_s) = (c_m, u_{c_m}))]  \\
	&\le \sum_{m'<m}\sum_{s\in T_{m'}} \bE_{ (P^{u^{c_m}})^{\otimes (s-1)}}[\1(i_s \neq c_{m'})] + \sum_{s<t, s\in T_m}  \bE_{ (P^{u^{c_m}})^{\otimes (s-1)}}[\1(i_s \neq c_{m}) + \1(j_s = u_{c_m})] \\
	&\le \sum_{m'=1}^M \sum_{s\in T_{m'}} \bE_{ (P^{u^{c_m}})^{\otimes (s-1)}}[\1(i_s \neq c_{m'})] + \sum_{s\in T_m} \bE_{ (P^{u^{c_m}})^{\otimes (s-1)}}[\1(j_s = u_{c_m})] \\
	&\le 4R_T + \sum_{s\in T_m} \bE_{ (P^{u^{c_m}})^{\otimes (s-1)}}[\1(j_s = u_{c_m})] ,
\end{align*}
where the last inequality is due to \eqref{eq:lowerbound_first} applied to $u^{c_m}$. The above two inequalities lead to
\begin{align*}
	R_T &\ge \frac{\Delta}{4} \cdot \sum_{m=1}^M \sum_{t\in T_m} \bE_{u\backslash \{u_{c_m} \}}\exp\left( -\frac{64\Delta^2}{3} \left(4R_T + \frac{1}{\alpha}\sum_{s\in T_m}\sum_{u_{c_m}=1}^\alpha \bE_{ (P^{u^{c_m}})^{\otimes (s-1)} }[\1(j_s = u_{c_m})]  \right) \right) \\
	&\stepe{\ge} \frac{\Delta}{4} \cdot \sum_{m=1}^M \sum_{t\in T_m} \bE_{u\backslash \{u_{c_m} \}}\exp\left( -\frac{64\Delta^2}{3} \left(4R_T + \frac{1}{\alpha}\sum_{s\in T_m} 1  \right) \right) \\
	&= \frac{\Delta T}{4} \cdot \exp\left( -\frac{64\Delta^2}{3} \left(4R_T + \frac{T}{\alpha M} \right) \right), 
\end{align*}
where (e) crucially uses the observation that $u^{c_m}$ does not depend on $u_{c_m}$, and $\sum_{u_{c_m}} \1(j_s = u_{c_m}) \le 1$. Plugging the expression of $\Delta = \sqrt{K/(16T)}$ and $K=(\alpha+1)M$, we have
\begin{align}\label{eq:lowerbound_final}
	R_T \ge \frac{\sqrt{KT}}{16}\exp\left( - \frac{16KR_T}{3T} - \frac{4(\alpha+1)}{3\alpha}\right) \ge \frac{\sqrt{KT}}{16e^3}\exp\left( - \frac{16KR_T}{3T}\right). 
\end{align}
To conclude that $R_T=\Omega(\sqrt{KT})$ from \eqref{eq:lowerbound_final}, we use the remaining assumption $K\le T^{1/3}$. If $R_T \le \sqrt{KT}/(32e^3)$, then
\begin{align*}
	\frac{\sqrt{KT}}{32e^3} \ge R_T \ge  \frac{\sqrt{KT}}{16e^3}\exp\left( - \frac{16KR_T}{3T}\right) \ge  \frac{\sqrt{KT}}{16e^3}\exp\left( - \frac{K^{3/2}}{6e^3T^{1/2}}\right) \ge \frac{\sqrt{KT}}{16e^3}\exp\left(-\frac{1}{6e^3}\right),
\end{align*}
which is a contradiction. Therefore, $R_T \ge \sqrt{KT}/(32e^3)$, and we are done. 
\end{proof} 

\begin{remark}\label{rem:lower_bound_insufficient}
	A close inspection of the proof reveals that the same lower bound construction works for the partially ordered contextual bandit structure possessed by the first-price auction as well. Therefore, Theorem \ref{thm.lower_bound} suggests that the partially ordered contextual bandit framework is insufficient to achieve an $\widetilde{O}(\sqrt{T})$ regret for first-price auctions under adversarial private values.
	In particular, the best lower bound given by Theorem~\ref{thm.lower_bound} for first-price auctions while balancing the discretization errors given in Lemma~\ref{lemma.approx_regret} is $\Omega(T^{2/3})$,
	where $M = K = T^{1/3}$.
	This motivates the need for exploiting additional structures in first-price auctions, which is the subject of Section \ref{sec.adversarial}. 
\end{remark}

%% file: adversarial.tex
\section{Learning in First-price Auctions with Adversarial Private Values}\label{sec.adversarial}
We now consider the more challenging problem where the 
private values are arbitrary and chosen by an oblivious adversary. 
Surprisingly, we show that an $\widetilde{O}(\sqrt{T})$ expected regret can still be achieved. 
To do so, we exploit the additional property in first-price auctions that rewards are correlated among actions, a property not present in the general partially ordered contextual bandits.
We provide a high-level sketch of the main ideas in Section \ref{subsec.idea} and detail the final (and more complicated) ML-IS-UCB policy in Section \ref{subsec.coUCB}, with the full analysis in Section \ref{subsec.coUCB_analysis}. 

\subsection{Main Idea: Correlated Rewards and Interval Splitting}\label{subsec.idea}
The lower bound in Theorem \ref{thm.lower_bound} shows that learning with adversarial private value sequences cannot be handled purely using a generic partially ordered contextual bandit structure. 
That is because, at the generality of a partially ordered contextual bandit, choosing one action reveals either complete information or no information about the mean reward of another action. 
However, as we shall see, the rewards of different actions in first-price auctions are actually \emph{correlated}, and therefore choosing an action may reveal \emph{partial} information of another action.

To illustrate this idea, we recall the key insights of Algorithm \ref{algo.partial_order} specialized to first-price auctions. Specifically, it makes use of the crucial fact that the outcome of bidding low prices provides full information for the rewards of bidding high prices.  Mathematically, observing $b_t$ and $\max\{b_t, m_t\}$ gives perfect knowledge of $\1(m_t\le b)$ for any $b\ge b_t$; therefore, bidding $b_t$ contributes to estimating the CDF $G(b)$ (and hence the mean reward of bidding $b$) for all $b\ge b_t$. However, when the value sequence is chosen adversarially to yield decreasing $b_t$'s over time, this structure becomes not helpful.  

Despite this hurdle, which is insurmountable in a generic partially ordered contextual bandit, our key insight in this section is that bidding high prices in fact provides \emph{partial} information for the rewards of bidding low prices: for any two bids $b_1$ and $b_2$, the random reward $(v-b_1)\1(m\le b_1)$ of bidding $b_1$ is \emph{correlated} with the random reward $(v-b_2)\1(m\le b_2)$ of bidding $b_2$. Therefore the former observation may help infer part of the latter, even if $b_1 > b_2$. We exploit this additional structure in first-price auctions via the following \emph{interval splitting} scheme. 

\subsubsection{An Interval-Splitting Scheme}

Since the only unknown part in the reward $R(v,b)$ in \eqref{eq.expected_reward} is $G(b)=\bP(m_t\le b)$, we may reduce the reward estimation problem to the estimation of $G$, or equivalently the complementary CDF (CCDF) $\bar{G}(b) = \bP(m_t>b)$. For $b<b'$, we write:
\begin{align}\label{eq.prob_decomposition}
\bar{G}(b) = \bP(b<m_t\le b') + \bP(m_t> b'). 
\end{align}
Now is the crucial insight in leveraging this probability decomposition in \eqref{eq.prob_decomposition}. On one hand, since $b'>b$, the second quantity $\bP(m_t>b')$ can be estimated within more precision, as more observations are available for $b'$ than $b$ due to the one-sided feedback. On the other hand, although the number of samples for estimating $\bP(b<m_t\le b')$ is the same as that for estimating $\bar{G}(b)$, the target probability to be estimated becomes smaller: $ \bP(b<m_t\le b') \le \bP(m_t>b)$. Here a smaller target probability helps because of the dispersion of the Binomial distribution: if we toss a coin $n$ times with head probability $p$, the empirical mean estimate will deviate from the truth by at most $O(\sqrt{p/n})$ in expectation. Consequently, all el the smaller the true probability $p$ is, the easier it is to obtain an accurate estimate. This indicates that the smaller $\bP(b<m_t\le b')$ can be estimated accurately even if the sample size is smaller. As such, the decomposition \eqref{eq.prob_decomposition} motivates us to estimate these probabilities separately to achieve a better accuracy: the second quantity $\bP(m_t> b')$ corresponds to the correlation in the rewards where a larger sample size for a larger bid is provided, and the first quantity $\bP(b<m_t\le b')$ has a smaller magnitude and therefore enjoys a smaller estimation error as well. 

Taking one step further, the decomposition \eqref{eq.prob_decomposition} can be extended to multiple intervals in a dynamic way over time. Specifically, let $\calP_t =\{b_1, \dots, b_t\}$ be the set of split points of the interval $[0,1]$ at the end of time $t$. Assuming distinct $b_t$'s, let $b_{(1)}< b_{(2)}<\cdots<b_{(t)}$ be the order statistic. 
With the convention $b_{(0)}= 0, b_{(t+1)} = 1$, it follows that for any candidate bid $b\in (b_{(s)}, b_{(s+1)}]$ with $0\le s\le t$, we have the following representation of the CCDF based on interval-splitting: 
\begin{align*}
\bar{G}(b) = \bP(b < m_t \le b_{(s+1)}) + \sum_{s'=s+1}^t \bP(b_{(s')} < m_t \le b_{(s'+1)}).
\end{align*}
We can then estimate $\bP(b_{(s')} < m_t \le b_{(s'+1)})$ using the empirical frequency computed from the past observations as follows (here $m_{(\ell)}$ is the HOB associated with $b_{(\ell)}$; note that $m_{(\ell)}$ is observed only if $m_{(\ell)} > b_{(\ell)}$): 
$$\widehat{\bP}(b_{(s')} < m_t \le b_{(s'+1)}) = \frac{1}{s'} \sum_{\ell=1}^{s'} 
 \1(m_{(\ell)} > b_{(\ell)}) \1(b_{(s')}  < m_{(\ell)}  \le b_{(s'+1)}).$$
Note that only the first $s'$ observations are used in the above estimate, for $m_{(\ell)}$ with $\ell > s'$ is either outside the interval $(b_{(s')},b_{(s'+1)}]$ when it is observed, or with an unknown membership in the interval $(b_{(s')},b_{(s'+1)}]$ when it is only known that $m_{(\ell)} \le b_{(\ell)}$. In both scenarios, these remaining observations cannot contribute to the estimation of $\bP(b_{(s')} < m_t \le b_{(s'+1)})$. Glancing at this line of reasoning, we see that the sample sizes in each partition are different and adaptively chosen based on the previous bids. In particular, the reward of bidding $b$ is estimated by combining the local estimates with different sample sizes. 

Finally, based on the estimated rewards, the bidder picks the bid $b_{t+1}$ for the time $t+1$, and update the partition as $\calP_{t+1} = \calP_t \cup \{b_{t+1}\}$. Repeating the previous estimation procedure on $\calP_{t+1}$ shows that, bidding a high price $b_{t+1}$ provides partial information for a lower price $b$ in the sense that an additive component of $\bar{G}(b)$, i.e., $\bar{G}(b_{t+1})$, can now be estimated using one more observation.

\subsubsection{Intuition of the UCB Construction}
The above interval-splitting scheme provides an efficient way of using past observations to estimate the CDF. We now integrate this estimation scheme into the bidding process to understand the intuition of the final bidding algorithm. First, we quantize the action space into $K=\lceil \sqrt{T}\rceil$ evenly spaced grid points $\calB = \{b^1,\cdots,b^K\}$, with $b^i = (i-1)/K$ and $b^{K+1}:= 1$. For each $(b^j, b^{j+1}]$, we count the number $n_j$ of past observations that contributes to estimating $p_j = \bP(b^j< m_t\le b^{j+1})$ (i.e. observations from $s\le t$ where $b_s\le b^j$) and from there compute the corresponding estimate $\widehat{p}_j$. From these estimates, we can then easily recover the CDF:
\begin{equation}\label{eq:cdf_estimate}
\widehat{G}(b^i) = 1 -\sum_{j=i}^K \widehat{p}_j.
\end{equation}



Now, at time $t$, after receiving $v_t$, we would ideally like to pick a bid $b^i$ in $\calB$ to maximize the expected reward $R(v_t, b^i) = (v_t - b^i) G(b^i)$ --- if we had known $G(\cdot)$. Of course, we do not know $G(\cdot)$,
and an available estimate comes from \eqref{eq:cdf_estimate}.
However, despite the best efforts that went into making that estimate efficient via the interval-splitting scheme,
distilled wisdom from the bandits literature would frown upon us if we had simply plug in $\widehat{G}(b^i)$ and pick a maximizing $b^i$ thereafter, because of the necessity of exploration. The celebrated idea of the upper confidence bound (UCB) algorithm in the bandit literature suggests that using an upper confidence bound of $\widehat{G}$ --- rather than $\widehat{G}$ itself --- could potentially be a good way to achieve exploration. 

As such, an immediate question --- the key to all UCB-based algorithms --- naturally arises: how should one choose the upper confidence bound in this non-parametric setting? Consider the following ideal scenario where each $\widehat{p}_j$ is a normalized Binomial random variable $n_j\widehat{p}_j \sim \mathsf{B}(n_j,p_j)$, and the estimates $\widehat{p}_j$ for different $j$ are negatively associated, as is the case in the classical multinomial models \cite{joag-dev1983}. In this case, we have
\begin{align*}
\var\left(\sum_{j=i}^K \widehat{p}_j \right) \le \sum_{j=i}^K\var\left( \widehat{p}_j \right) = \sum_{j=i}^K \frac{p_j(1-p_j)}{n_j} \le \sum_{j=i}^K \frac{p_j}{n_j}.
\end{align*}
This suggests the use of the following UCB algorithm (with $\widehat{G}(b^j)$ given in \eqref{eq:cdf_estimate}): 
\begin{align}\label{eq.coUCB}
				b_t = \arg\max_{b^i\in \calB} \quad (v_t - b^i)\left[\widehat{G}(b^j)+ \gamma \left(\sqrt{\sum_{j=i}^K \frac{p_j}{n_j} } + \frac{1}{n_i}\right) \right]. 
\end{align}
In the ideal scenario where $n_j\widehat{p}_j\sim \mathsf{B}(n_j,p_j)$ is exactly Binomial, the above upper confidence bound is valid with high probability by the Bernstein inequality (cf. Lemma \ref{lemma.bennett}), with tuning parameter $\gamma$ at most polylogarithmic in $T$. 

However, there are two catches of using the confidence bound \eqref{eq.coUCB}. First, we do not know $p_j$, so we need to replace it by the empirical estimator $\widehat{p}_j$ in \eqref{eq.coUCB}. Second, and more important, the above ideal scenario does not hold since each estimator $\widehat{p}_j$ is \emph{not} a normalized Binomial random variable. Specifically, the estimator at time $t$ is
\begin{align*}
\widehat{p}_j^t = \frac{\sum_{s\le t} \1(b_s\le b^j)\1(m_s\in (b^j, b^{j+1}])}{\sum_{s\le t} \1(b_s\le b^j)}. 
\end{align*}
As such, conditioned on the bids $(b_1,\cdots,b_t)$, the HOBs $(m_1,\cdots,m_t)$ are no longer mutually independent: the choice of $b_t$ in \eqref{eq.coUCB} depends on the previous probability estimates, and thus depends on $(m_s)_{s<t}$. As a result, we do not even have the unbiased property, i.e. $\bE[\widehat{p}_j^t] \neq p_j$ in general. This issue also arises in the proof of Lemma \ref{lemma.concentration}, where each individual estimate $\widehat{p}_j^t$ can be handled using the theory of self-normalized martingales. However, the martingale approach breaks down in the current scenario where we need to handle the concentration property of the sum $\sum_{j\ge i} \widehat{p}_j^t$ due to the complex dependencies among $\widehat{p}_j^t$'s. To circumvent the dependence issue, we need a multi-layer algorithm to \emph{force} each estimator $\widehat{p}_j^t$ to be normalized Binomial and show that the desired theoretical guarantee on regret can still be obtained, which we discuss next.

\subsection{A Master Algorithm: The ML-IS-UCB Policy}\label{subsec.coUCB}

We analyze a master algorithm that is a multi-level version of the UCB intuition, which we call the Multi-Level Interval-Splitting UCB (ML-IS-UCB) policy; see Algorithm \ref{algo.ml-ucb}. The high-level idea is that instead of using all past observations up to time $t$ for computation, the estimator $\widehat{p}_j^t$ is computed by a subset of the history $\Phi(t)\subseteq [t]$ with the key property that \emph{$(m_s)_{s\in \Phi(t)}$ are mutually independent conditioned on $(b_s)_{s\in \Phi(t)}$}. This theoretical device of forcing conditional independence by throwing away a subset of data when performing estimation was first proposed in~\cite{auer2002using}, where a baseline contextual bandits algorithm LinREL is first given for linear contextual bandits (without regret guarantees due to the same dependency issue), and then a master algorithm that uses LinREL on a subset of past data is constructed to force conditional independence, giving an explicit regret bound. Subsequently, this master scheme has been used widely for different contextual bandits algorithm as a way to force conditional independence, such as the SupLinUCB algorithm in \cite{chu2011contextual} for linear contextual bandits with changing contexts and \cite{LLZ2017} for generalized linear contextual bandits. The main idea of forcing the conditional independence is the following: for a given action, although the estimated reward is formed after observing the empirical rewards, \emph{the accuracy of the estimate can be formed without observing the empirical rewards}. For example, if $X_1,\cdots,X_n\sim \calN(\mu,1)$ with an unknown mean $\mu$, the mean estimate $\overline{X}_n = n^{-1}\sum_{i=1}^n X_i$ depends on the sample $(X_1,\cdots,X_n)$, but the standard deviation $n^{-1/2}$ of $\overline{X}_n$ is independent of the sample. In other words, the randomness in the accuracy parameter (used to form the UCB) is decoupled with the randomness in the estimated reward. Our ML-IS-UCB algorithm follows the same philosophy, with the following distinctions from \cite{auer2002using, chu2011contextual,LLZ2017}:
\begin{itemize}
	\item the construction of the upper confidence bound is on a case-by-case basis depending on the underlying UCB algorithm, and our bound is designed in a nonparametric way; 
	\item the accuracy parameter in \eqref{eq:width} depends on the estimated rewards $\widehat{p}_j^0$ \emph{on a held-out sample}, a nature that is not present in the prior work. 
\end{itemize}

\begin{algorithm}[!htbp]
	\footnotesize
	\caption{Multi-Level Interval-Splitting Upper Confidence Bound (ML-IS-UCB) Policy	\label{algo.ml-ucb}}
	\textbf{Input:} Time horizon $T$; action set $\mathcal{B}=\{b^1,\cdots,b^K\}$ with $b^i = (i-1)/K$ and $K=\lceil \sqrt{T} \rceil$; number of levels $L = \lceil \log_2 T\rceil$; tuning parameter $\gamma>0$. \\
	\textbf{Output:} A resulting policy $\pi$. \\
	\textbf{Initialization:} Set $T_0 \gets (L+1)\cdot \lceil \sqrt{T} \rceil$, and the bidder bids $b_t = b^1 = 0$ for all $1\le t\le T_0$;\\
	Initialize $\Phi^{\ell}(T_0) \gets \{\ell \cdot \lceil \sqrt{T} \rceil +1, \cdots, (\ell+1) \cdot \lceil \sqrt{T} \rceil  \}$ for each $\ell=0,1,2,\cdots,L$; \\
	Initialize $\widehat{p}_i^0 \gets |\Phi^0(T_0)|^{-1}\sum_{t\in \Phi^0(T_0)}\1(b^i< m_t\le b^{i+1})$ for each $i\in [K]$.\\
	\For{$t = T_0+1,T_0+2,\cdots,T$}{
		The bidder receives the private value $v_t\in [0,1]$; \\
		The bidder initializes $\calB_t^0 \gets [K]$; \\
		\For{$\ell = 1,2,\cdots,L $}{
			For all $i\in [K]$, the bidder computes the number of observations for the action $b^i$ at level $\ell$:
			\begin{align}\label{eq:sample_size_level}
			n_{i,t}^{\ell} \gets \sum_{s\in \Phi^{\ell}(t-1)} \1(b_s \le b^i). 
			\end{align}\\
			For all $i\in \calB_t^{\ell-1}$, the bidder computes the following width $w_{i,t}^{\ell}$ of the action $b^i$: 
			\begin{align}\label{eq:width}
			w_{i,t}^{\ell} \gets  \gamma\left(\sqrt{2\log(LKT)\sum_{j\ge i}\frac{1}{n_{j,t}^{\ell}} \left(\widehat{p}_j^0 + \frac{\gamma \log(LKT)}{\lceil \sqrt{T} \rceil}\right)  }+ \frac{\log(LKT)}{n_{i,t}^{\ell}} \right). 
			\end{align}\\
			\eIf{there exists $i\in \calB_t^{\ell-1}$ such that $w_{i,t}^{\ell} > 2^{-\ell}$}{
				The bidder bids $b_t = b^i$ for an arbitrary $i\in \calB_t^{\ell-1}$ with $w_{i,t}^{\ell} > 2^{-\ell}$;\\
				The bidder updates $\Phi^{\ell}(t) \gets \Phi^{\ell}(t-1) \cup \{t\} $ and $\Phi^{\ell'}(t)\gets \Phi^{\ell'}(t-1)$ for all $\ell'\neq \ell$;\\
				\textbf{Break the inner for loop.}}{
				For all $i\in [K]$, the bidder computes the empirical probability of each interval: 
				\begin{align}\label{eq:probability_level}
				\widehat{p}_{i,t}^{\ell} \gets \frac{1}{n_{i,t}^{\ell}}\sum_{s\in \Phi^{\ell}(t-1)} \1(b_s\le b^i)\1(b^i< m_s\le b^{i+1}); 
				\end{align}\\
				For all $i\in \calB_t^{\ell - 1}$, the bidder computes the average rewards at level $\ell$: 
				\begin{align}\label{eq:reward_level}
				\widehat{r}_{i,t}^{\ell} \gets (v_t - b^i)\cdot \left(1 - \sum_{j\ge i} \widehat{p}_{j,t}^{\ell} \right); 
				\end{align}\\
				Eliminate bad actions at level $\ell$: 
				\begin{align}\label{eq:active_set_level}
				\calB_{t}^{\ell} \gets \left\{i\in \calB_{t}^{\ell-1}: \widehat{r}_{i,t}^{\ell} + w_{i,t}^{\ell} \ge \max_{j\in \calB_{t}^{\ell-1}}\left(\widehat{r}_{j,t}^{\ell} + w_{j,t}^{\ell}\right) - 2\cdot 2^{-\ell}  \right\}.
				\end{align}\\
				\textbf{Continue the inner for loop.}
			}
		}
		The bidder observes $\max\{b_t, m_t\}$; 
	}
\end{algorithm}

Next we describe the algorithm in detail. In ML-IS-UCB, we split the entire time horizon $[T]$ into $L$ different levels $\Phi^1(T),\cdots,\Phi^L(T)$ in a sequential manner, with $L = \lceil \log_2 T\rceil$. These sets form a partition of $[T]$, and therefore the subsets $\Phi^{\ell}(t) = \Phi^{\ell}(T)\cap [t]$ also form a partition of the current history $[t]$. We will call the collection of HOBs $(m_t)_{t\in \Phi^{\ell}(T)}$ as \emph{level-$\ell$ HOBs}. Heuristically, the level-$\ell$ HOBs are responsible for the level-$\ell$ estimation of the rewards for each candidate action surviving through this level, as well as a \emph{certificate} stating whether the estimation accuracy for each candidate action is within $2^{-\ell}$ or not. Specifically, at each time $t\in [T]$, we associate three objects to every level $\ell$ and every candidate bid $b^i$: 
\begin{enumerate}
	\item The certificate, or the \emph{width} $w_{i,t}^{\ell}$: The certificate is the first quantity to compute when entering each level and measures the reward estimation accuracy of the candidate bid $b^i$. Consequently, if $w_{i,t}^{\ell}$ is small for some action $i\in [K]$, then the level-$\ell$ reward estimate of the action $i$ is expected to be accurate, \emph{even before the bidder knows what the level-$\ell$ estimates are}. In particular, the computation of the certificate does not depend on $(m_s)_{s\in \Phi^{\ell}(T)}$. 
	
	Specifically, the certificate is set to be the width of the upper confidence bound in \eqref{eq.coUCB}, with several modifications. First, the number of observations $n_{i,t}^{\ell}$ for the $i$-th interval is counted only on past level-$\ell$ bids, as shown in \eqref{eq:sample_size_level}. Second, the probability estimate for each interval appearing in \eqref{eq.coUCB} (to estimate the variance) is chosen to be the estimates $\widehat{p}_{i}^{0}$ outputted by the held-out sample of size $\lceil \sqrt{T} \rceil$ on level $0$. Third, some additional terms are added to the width mainly for technical purposes. The final expression of $w_{i,t}^{\ell}$ is displayed in \eqref{eq:width}, and it is easy to see that the width does not depend on the level-$\ell$ HOBs $(m_s)_{s\in \Phi^{\ell}(T)}$. 
	\item The estimates $\widehat{p}_{i,t}^{\ell}$ and $\widehat{r}_{i,t}^{\ell}$: There are two possibilities after the certification step. If there is a candidate action with width at least $2^{-\ell}$, then the bidder stops here and chooses that action (if there are multiple such actions choose an arbitrary one). Also, the bidder assigns level $\ell$ to the time point $t$. This is an exploration step, where no estimates are performed at level $\ell$. Conversely, if all candidate actions have width smaller than $2^{-\ell}$, the bidder will perform the estimation and move to level $\ell+1$. Specifically, the bidder computes the empirical probability of each interval using only level-$\ell$ HOBs (cf. \eqref{eq:probability_level}), and consequently computes the reward estimates for each candidate action based on the above empirical probabilities (cf. \eqref{eq:reward_level}, which implements the interval-splitting scheme). Note that if the certificates are accurate, then both estimates are guaranteed to be $O(2^{-\ell})$-close to the truth. 
	\item The candidate action set $\calB_t^\ell$: After the bidder computes the estimates, she makes the exploitation step and eliminates probably bad actions from the candidates. Specifically, she updates the candidate set from $\calB_{t}^{\ell-1}$ to $\calB_t^{\ell}$ based on a UCB-type rule, which removes all actions whose upper confidence bound $\widehat{r}_{i,t}^{\ell} + w_{i,t}^{\ell}$ is smaller than the best one by a margin of at least $2\cdot 2^{-\ell}$ (cf. \eqref{eq:active_set_level}). The choice of the margin ensures two important properties of the elimination rule. First, the best action is never eliminated with high probability. Second, if any action in the candidate set $\calB_t^{\ell}$ is selected, the instantaneous regret is at most $O(2^{-\ell})$. After the update on the candidate action set, the bidder moves to level $\ell+1$ and repeat the above procedures. 
\end{enumerate}

In summary, instead of computing the reward estimates and widths simultaneously, ML-IS-UCB decouples the certification phase and the estimation phase. Specifically, the widths are computed first at each level without knowing the observations at that level, and the bidder performs estimation and \emph{exploits} a good action only if she is confident that the estimates she is about to carry out are expected to be accurate (as quantified by the previous widths). Otherwise she \emph{explores} and chooses any action with a large width, hoping that the width of the selected action will shrink below the desired level next time. With this decoupling, $(m_t)_{t\in \Phi^{\ell}(T)}$ are now mutually independent conditioned on the level-$\ell$ bids $(b_t)_{t\in \Phi^{\ell}(T)}$ for each level: 

\begin{lemma}\label{lemma:independence}
	Under an oblivious adversary, for each $\ell\in [L]$, the level-$\ell$ HOBs $(m_t)_{t\in \Phi^{\ell}(T)}$ are independent of the level-$\ell$ bids $(b_t)_{t\in \Phi^{\ell}(T)}$. Consequently, the level-$\ell$ HOBs $(m_t)_{t\in \Phi^{\ell}(T)}$ are mutually independent conditioning on the level-$\ell$ bids $(b_t)_{t\in \Phi^{\ell}(T)}$. 
\end{lemma}
\begin{proof} 
	Fix any $t\in \Phi^{\ell}(T)$. Since the time point $t$ is classified as level $\ell$, the bidder at time $t$ exploits in the first $\ell-1$ levels and explores in the last level. Hence, the bid $b_t$ is determined by the widths $w_{i,t}^{\ell}$ for all $i\in \calB_t^{\ell-1}$, which in turn depends on $n_{i,t}^{\ell}$ in \eqref{eq:sample_size_level}, the previous estimates $\widehat{p}_{i,t}^{\ell-1}, \widehat{r}_{i,t}^{\ell-1}$ in \eqref{eq:probability_level}, \eqref{eq:reward_level}, and the previous candidate set $\calB_{t}^{\ell-1}$ in \eqref{eq:active_set_level}. Continuing the above process until $\ell=0$ shows that $b_t$ only depends on the quantities $(b_s: s\in \cup_{\ell'=0}^{\ell} \Phi^{\ell'}(t-1))$, $(m_s: s\in \cup_{\ell'=0}^{\ell-1} \Phi^{\ell'}(t-1))$ and $v_t$. If we further expand the dependence of the previous bids $b_s$ explicitly, we conclude that $b_t$ is  determined by $(m_s: s\in \cup_{\ell'=0}^{\ell-1} \Phi^{\ell'}(t-1))$ and $(v_s: s\in \cup_{\ell'=0}^{\ell} \Phi^{\ell'}(t))$. Consequently, the collection of level-$\ell$ bids depends solely on $(m_t: t\in \cup_{\ell'=0}^{\ell-1} \Phi^{\ell'}(T))$ and $(v_t: t\in \cup_{\ell'=0}^{\ell} \Phi^{\ell'}(T))$, which by the i.i.d. assumption of $(m_t)_{t\in[T]}$ and the definition of an oblivious adversary yields the claimed independence.
\end{proof}

Several other important details of Algorithm \ref{algo.ml-ucb} are in order. First, the bidder always bids zero in the first $T_0=O(\sqrt{T}\log T)$ rounds, and attributes the resulting observations evenly to $L$ levels. This pure-exploration step ensures that the number of observations at each level is at least $\sqrt{T}$, so that all estimates have enough accuracy even at the beginning. Moreover, this step only leads to a total regret at most $O(\sqrt{T}\log T)$. Second, we also define the level-$0$ observations only consisting of the first $\sqrt{T}$ observations where the bidder again consistently bids zero. These level-$0$ observations are used to form the initial probability estimates involved in the computation of the level-$1$ certificates. Finally, the loop in Algorithm \ref{algo.ml-ucb} will break before $\ell=L$, for \eqref{eq:width} shows that $w_{i,t}^{\ell} \ge 1/\sqrt{T}$ and therefore the widths $w_{i,t}^{\ell}$ cannot be smaller than $2^{-L} \le 1/T$.  Hence the algorithm always terminates.
Note that if Break is executed in Algorithm \ref{algo.ml-ucb}, it breaks out of the inner for loop and hence
$l$ restarts at 1 at the next iteration of the outer for loop.
We are now ready to characterize the learning performance of ML-IS-UCB. 

\begin{theorem}\label{thm.coUCB}
	Let the private values $v_1,\cdots,v_t\in [0,1]$ be a sequence chosen by an oblivious adversary. For $\gamma\ge 3$, the regret of the \emph{ML-IS-UCB} policy satisfies
	\begin{align*}
	\bE[R(\pi^{\text{\rm ML-IS-UCB}};v)] \le 4 + (L+4)\cdot \lceil \sqrt{T}\rceil + 80\gamma^{3/2} \log(LKT)(1+\log T)\cdot L\sqrt{T},
	\end{align*}
where we recall that $L = \lceil \log_2 T\rceil$ and $K = \lceil \sqrt{T}\rceil$ in the algorithm. 
\end{theorem}
Theorem \ref{thm.coUCB} implies Theorem \ref{thm.arbitrary}  and shows an $\widetilde{O}(\sqrt{T})$ regret under any private values. 

\subsection{Analysis of the ML-IS-UCB Policy}\label{subsec.coUCB_analysis}
This section is devoted to the proof of Theorem \ref{thm.coUCB}. The proof breaks into two parts. First, we show that the width $w_{i,t}^{\ell}$ constructed in \eqref{eq:width} is a reliable certificate for the estimates, i.e. $|\widehat{r}_{i,t}^{\ell} - r_{i,t}| \le w_{i,t}^{\ell}$ holds simultaneously for all $i\in [K], t\in [T], \ell\in [L]$ with high probability. Hence, the instantaneous regret at each time $t\in \Phi^{\ell}(T)$ is at most $O(2^{-\ell})$. Second, we provide an upper bound of $|\Phi^{\ell}(T)|$ for each $\ell = 1,2,\cdots,L$, showing that the total regret of the ML-IS-UCB policy is at most $\widetilde{O}(\sqrt{T})$.  

Thanks to Lemma \ref{lemma:independence}, the following empirical probability 
\begin{align*}
\widehat{p}_{i,t}^{\ell} = \frac{\sum_{s\in \Phi^{\ell}(t-1)} \1(b_s\le b^i)\1(b^i< m_s\le b^{i+1})}{\sum_{s\in \Phi^{\ell}(t-1)} \1(b_s\le b^i)}
\end{align*}
in \eqref{eq:probability_level} is a normalized Binomial random variable if we condition on all $(b_s: s\in \Phi^{\ell}(t-1))$, with sample size $n_{i,t}^{\ell}$ and success probability $p_i \triangleq \bP(m_1\in (b^i, b^{i+1}]) = G(b^{i+1}) - G(b^i)$. As a result, the following concentration inequality holds for all reward estimates $\widehat{r}_{i,t}^{\ell}$. 

\begin{lemma}\label{lemma.UCB}
	Let $r_{i,t} = R(b^i,v_t)$ be the expected reward of bidding $b^i\in \calB$ at time $t$, and $\widehat{r}_{i,t}^{\ell}$ be the level-$\ell$ estimate of $r_{i,t}$ constructed in \eqref{eq:reward_level}. For $\gamma\ge 3$, with probability at least $1-4T^{-3}$, the following inequality holds simultaneously for all $i\in [K], \ell\in [L]$, and $T_0<t\le T$: 
	\begin{align}\label{eq.lemma_concentration}
	\left|\widehat{r}_{i,t}^{\ell} - r_{i,t} \right| &\le w_{i,t}^{\ell} =  \gamma\left(\sqrt{2\log(LKT)\sum_{j\ge i}\frac{1}{n_{j,t}^{\ell}} \left(\widehat{p}_j^0 + \frac{\gamma \log(LKT)}{\lceil \sqrt{T} \rceil}\right)  }+ \frac{\log(LKT)}{n_{i,t}^{\ell}} \right) \\
	&\le  \gamma\left(\sqrt{6\log(LKT)\sum_{j\ge i}\frac{1}{n_{j,t}^{\ell}} \left(p_j + \frac{\gamma \log(LKT)}{\lceil \sqrt{T} \rceil}\right)  }+ \frac{\log(LKT)}{n_{i,t}^{\ell}} \right). 
	 \label{eq.variance_bound}
	\end{align}
\end{lemma}

The proof of Lemma \ref{lemma.UCB} requires a careful application of the Bernstein inequality to the estimation of both the true rewards $r_{i,t}$ and the true variance $\sum_{j\ge i} p_j/n_{j,t}^{\ell}$, which is relegated to the Appendix. Specifically, Lemma \ref{lemma.UCB} shows that the widths constructed in \eqref{eq:width}, albeit independent of the level-$\ell$ HOBs, are reliable certificates of the future reward estimates using the level-$\ell$ HOBs. The following Lemma \ref{lemma:consequence} summarizes some desirable consequences of Lemma \ref{lemma.UCB}. 

\begin{lemma}\label{lemma:consequence}
Assume that the inequalities of Lemma \ref{lemma.UCB} hold uniformly. Then for each time $t> T_0$, let $i^\star(t) = \arg\max_{i\in [K]} r_{i,t}$ be the index of the optimal bid in the set $\calB$, and $t \in \Phi^l(T)$. Then $i^\star(t)$ is never eliminated from the candidate action set at time $t$, and the instantaneous regret of time $t$ is at most $8\cdot 2^{-\ell}$. 
\end{lemma}
\begin{proof}
We first show that the optimal action $i^\star(t)$ is never eliminated. In fact, since the bidder always exploits before level $\ell$, the certificates satisfy that $w_{i,t}^{\ell'} \le 2^{-\ell'}$ for all $1\le \ell'<\ell$ and $i\in \calB_t^{\ell'-1}$. Hence, for each $1\le \ell' < \ell$, Lemma \ref{lemma.UCB} gives the following chain of inequalities: 
\begin{align*}
\widehat{r}_{i^\star(t),t}^{\ell'} + w_{i^\star(t),t}^{\ell'} &\ge  r_{i^\star(t),t} \ge \max_{i\in \calB_{t}^{\ell'-1}} r_{i,t} \ge \max_{i\in \calB_{t}^{\ell'-1}} (\widehat{r}_{i,t}^{\ell'} - w_{i,t}^{\ell'}) \\
&\ge \max_{i\in \calB_{t}^{\ell'-1}} (\widehat{r}_{i,t}^{\ell'} + w_{i,t}^{\ell'}) - 2\max_{i\in \calB_{t}^{\ell'-1}}w_{i,t}^{\ell'} \\
&\ge \max_{i\in \calB_{t}^{\ell'-1}} (\widehat{r}_{i,t}^{\ell'} + w_{i,t}^{\ell'}) - 2\cdot 2^{-\ell'},
\end{align*}
implying that $i^\star(t)$ is not eliminated by the rule \eqref{eq:active_set_level}. 

For the second statement, let $i\in [K]$ be the action index chosen by the bidder at time $t$. Since the instantaneous regret is upper bounded by one, the statement is vacuous if $\ell=1$, and we may assume that $\ell\ge 2$. As a result, the action $i$ passes the test \eqref{eq:active_set_level} at level $\ell-1$, implying that (note that we have established the fact that $i^\star(t)$ is not eliminated)
\begin{align*}
r_{i^\star(t),t} - r_{i,t} &\le \widehat{r}_{i^\star(t),t}^{\ell -1} - \widehat{r}_{i,t}^{\ell-1} + w_{i^\star(t),t}^{\ell -1}  + w_{i,t}^{\ell -1} \\
&\le (w_{i,t}^{\ell-1} - w_{i^\star(t),t}^{\ell-1} + 2\cdot 2^{-(\ell-1)}) + w_{i^\star(t),t}^{\ell -1}  + w_{i,t}^{\ell -1} \\
&\le 2\cdot 2^{-(\ell-1)} + 2\cdot 2^{-(\ell-1)} = 8\cdot 2^{-\ell},
\end{align*}
which establishes the claimed regret bound.  
\end{proof}

As Lemma \ref{lemma:consequence} upper bounds each instantaneous regret in terms of the assigned level at each time, it offers a useful tool to upper bound the total regret. Specifically, since the total regret during the first $T_0$ rounds is at most $T_0$, a double counting argument for $T_0<t\le T$ leads to
\begin{align}\label{eq:double_counting}
\sum_{t=1}^T \left( \max_{i\in [K]}r_{i,t} - r_{i(t),t}\right) \le T_0 + 8 \sum_{\ell=1}^L 2^{-\ell}\cdot |\Phi^{\ell}(T) \cap \{t: T_0<t\le T\} |
\end{align}
provided that the high probability event in Lemma \ref{lemma.UCB} holds, and $i(t)\in [K]$ is the action taken by the policy at time $t$. Hence, to upper bound the total regret shown in \eqref{eq:double_counting}, it remains to upper bound the cardinality of $\Phi^{\ell}(T)\cap \{t: T_0<t\le T\}$, especially for small $\ell$. Intuitively, such an upper bound is possible as an exploration always shrinks the widths of the confidence band, so the bidder cannot explore too much at any given level. The following lemma makes the above heuristics formal. 
\begin{lemma}\label{lemma:exploration}
For each $\ell\in [L]$, the following inequality holds: 
\begin{align*}
2^{-\ell}\cdot |\Phi^{\ell}(T) \cap \{t: T_0<t\le T\} | \le 10\gamma^{3/2}\log(LKT)(1+\log T)\cdot \sqrt{T}.
\end{align*}
\end{lemma}
\begin{proof}
For any $T_0<t\le T$, the definition of $t\in \Phi^{\ell}(T)$ implies that the width of the chosen action is at least $2^{-\ell}$, i.e. $w_{i(t),t}^{\ell} > 2^{-\ell}$. Hence, 
\begin{align*}
&2^{-\ell}\cdot |\Phi^{\ell}(T) \cap \{t: T_0<t\le T\} | \\
 &\le \sum_{t> T_0: t\in \Phi^{\ell}(T)} w_{i(t),t}^{\ell} \\
&\stepa{\le}  \sum_{t> T_0: t\in \Phi^{\ell}(T)} \gamma\left(\sqrt{6\log(LKT)\sum_{j\ge i(t)}\frac{1}{n_{j,t}^{\ell}} \left(p_j + \frac{\gamma \log(LKT)}{\lceil \sqrt{T} \rceil}\right)  }+ \frac{\log(LKT)}{n_{i(t),t}^{\ell}} \right) \\
&\stepb{\le} \gamma\sqrt{6|\Phi^{\ell}(T)|\log(LKT)}\cdot \sqrt{\sum_{t>T_0: t\in \Phi^{\ell}(T)}\sum_{j\ge i(t)}  \frac{1}{n_{j,t}^{\ell}} \left(p_j + \frac{\gamma \log(LKT)}{\lceil \sqrt{T} \rceil}\right) } +\sum_{t>T_0: t\in \Phi^{\ell}(T)} \frac{ \gamma\log(LKT) }{n_{i(t),t}^{\ell}} \\
&\stepc{\le} \gamma\left(\sqrt{6T\log(LKT)(1+\log T)\left( 1 + \frac{\gamma K\log(LKT)}{\lceil \sqrt{T} \rceil} \right)} +\log(LKT)(1+\log T)\cdot K  \right),
\end{align*}
where (a) follows from \eqref{eq.variance_bound} in Lemma \ref{lemma.UCB}, (b) is due to the Cauchy--Schwartz inequality, and (c) is due to Lemma \ref{lemma.combinatorial} in the appendix (a key combinatorial lemma) and the simple fact that $|\Phi^{\ell}(T)|\le T$. The rest of the proof follows from simple algebra and the choice $K = \lceil \sqrt{T} \rceil$. 
\end{proof}

\begin{remark}
	The presence of $p_j$ in the regret upper bound plays a key role in a small total sum of the regret contributions from all intervals, highlighting the necessity of tight upper confidence bounds during the interval splitting. Note that using the loose bound where $p_j$'s are replaced by $1$, applying the first inequality in Lemma \ref{lemma.combinatorial} renders an $\widetilde{O}(\sqrt{KT})$ final regret instead of $\widetilde{O}(\sqrt{T})$. 
\end{remark}

Finally, combining Lemma \ref{lemma:exploration} and the inequality \eqref{eq:double_counting}, the following upper bound on the total regret holds provided that the high-probability events in Lemma \ref{lemma.UCB} holds: 
\begin{align*}
\sum_{t=1}^T \left( \max_{i\in [K]}r_{i,t} - r_{i(t),t}\right) \le (L+1)\cdot \lceil \sqrt{T}\rceil + 80\gamma^{3/2} \log(LKT)(1+\log T)\cdot L\sqrt{T}. 
\end{align*}
Now taking into account the failure probability in Lemma \ref{lemma.UCB}, and following the same quantization argument in Lemma \ref{lemma.approx_regret}, the desired upper bound of Theorem \ref{thm.coUCB} follows from the choices of parameters $K=\lceil \sqrt{T} \rceil$ and $L = \lceil \log_2 T\rceil$.  

%% file: applications.tex
\section{Extensions and Discussions}\label{sec:discussion}
We conclude the paper with two more related discussions that enrich our presentation so far.

\subsection{Reversed One-sided Feedback}\label{subsec.reversed_feedback}
Throughout the previous sections, we assumed that the transaction price $\max\{b_t, m_t\}$ is revealed to every bidder, naturally implying that all losers are able to observe $m_t$. There could be a reversed one-sided feedback as well, where the knowledge of $m_t$ is only revealed to the winner. This knowledge is known as the ``minimum-bid-to-win'' at the Google Ad Exchange. We claim that our regret bounds remain unchanged in Theorem \ref{thm.stochastic} and Theorem \ref{thm.arbitrary} under this reversed one-sided feedback. 

The proof of the new Theorem \ref{thm.stochastic} is straightforward: under the reversed one-sided feedback, the feedback graph across actions becomes a reversed chain in Section \ref{subsec.application_auction}, i.e. there is an edge from $b$ to $b'$ if and only if $b>b'$. In other words, bidding high prices provides information for bidding low prices. Since Lemma \ref{lemma.monotonicity} holds independent of the feedback structure, the partial order over the contexts also becomes reversed: $v_1 \preceq_\calC v_2 \Leftrightarrow v_1 \ge v_2$. Under the above definitions, Definitions \ref{def.contextual_feedback_cross} and \ref{assump:partial_order} still hold with $\alpha=\beta=1$, and Theorem \ref{thm.stochastic} is a direct consequence of Theorem \ref{thm:partial_order}.

The proof of the new Theorem \ref{thm.arbitrary} requires more modifications of Algorithm \ref{algo.ml-ucb}: 
\begin{itemize}
	\item in the initialization phase $t\in [T_0]$, the bidder should always bid $b_t = 1$; 
	\item the definition of $n_{i,t}^{\ell}$ in \eqref{eq:sample_size_level} is changed into $n_{i,t}^{\ell} = \sum_{s\in \Phi^{\ell}(t-1)} \1(b_s\ge b^{i+1})$; 
	\item in the definition of $w_{i,t}^{\ell}$ in \eqref{eq:width}, the sum $\sum_{j\ge i}$ is replaced by $\sum_{j<i}$; 
	\item in the definition of $\widehat{p}_{i,t}^{\ell}$ in \eqref{eq:probability_level}, the indicator $\1(b_s\le b^i)$ is replaced by $\1(b_s\ge b^{i+1})$;
	\item the definition of $\widehat{r}_{i,t}^{\ell}$ in \eqref{eq:reward_level} becomes $\widehat{r}_{i,t}^{\ell} = (v_t - b^i)\cdot \sum_{j<i} \widehat{p}_{j,t}^\ell$;
	\item in the last line, the bidder now observes $\min\{b_t, m_t\}$. 
\end{itemize}
The proof of Theorem \ref{thm.coUCB} could then be carried out symmetrically, which leads to the same result of Theorem \ref{thm.arbitrary} under the reversed one-sided feedback.

\subsection{Impossibility Result under non-iid HOBs}\label{subsec.lower_bound}

So far we have assumed that the HOB $m_t$ is stochastic and has established
$\widetilde{O}(\sqrt{T})$ regret bounds regardless of whether $v_t$ is stochastic or adversarial.
What if $m_t$ is adversarial, or just non-iid? Although the general case is a difficult question that is beyond the scope of the current paper, we can provide an impossibility result even if the distribution of $m_t$ only slightly depends on $v_t$. 


Consider the following scenario where the HOB distribution of $m_t$ is a thresholding function of the private value $v_t$, i.e. $m_t \sim P$ if $v_t < v^\star$ and $m_t \sim Q$ if $v_t \ge v^\star$. In addition, we even assume that the distributions $(P, Q)$ are revealed to the bidder before bidding takes place, and a full-information feedback is always provided to the bidder, i.e. $m_t$ is always observed at the end of time $t$. The only crucial parameter which is unknown to the learner is the choice of the threshold $v^\star$. Let $b^\star_P(v) \triangleq \arg\max_{b\in [0,1]} \bE_{m\sim P}[(v-b)\1(b\ge m)]$ be the best bidding price under distribution $P$, and $b_Q^\star(v)$ be defined similarly. The regret of the learner is then defined as
\begin{align*}
	&\sum_{t=1}^T \left( \1(v_t < v^\star)\cdot \bE_{m_t\sim P}[(v_t-b_P^\star(v_t))\1(b_P^\star(v_t)\ge m_t) - (v_t-b_t)\1(b_t\ge m_t)] \right. \\
	& \qquad \left. + \1(v_t \ge v^\star)\cdot \bE_{m_t\sim Q}[(v_t-b_Q^\star(v_t))\1(b_Q^\star(v_t)\ge m_t) - (v_t-b_t)\1(b_t\ge m_t)] \right),
\end{align*}
where the oracle knows $v^\star$ and makes the optimal bid at every round. 

The above example is arguably the simplest scenario where the distribution of $m_t$ is non-iid and allowed to depend on $v_t$. Unfortunately, we have the following negative result showing that the worst-case regret is $\Omega(T)$ if $(P, Q, v^\star)$ and the private values are properly chosen.

\begin{theorem}\label{thm:monotone}
For any policy $\pi$, there exist $(P,Q,v^\star)$ and an oblivious adversarial sequence $(v_t)_{t\in [T]}$ such that $\pi$ has an expected regret at least $\Omega(T)$. 
\end{theorem}

\begin{proof}
The choices of $P$ and $Q$ are simple: we take $m_t \equiv 0$ under $P$ and $m_t \equiv 1/8$ under $Q$. To construct the worst-case $(v^\star, (v_t)_{t\in [T]})$, we impose the following joint distribution on $(v^\star, (v_t)_{t\in [T]})$. First, let the marginal distribution of $v^\star$ be the uniform distribution on the following set: 
\begin{align*}
	\left\{f(u_1,\cdots,u_{T}) \triangleq  \frac{1}{2} + \frac{1}{4}\sum_{t=1}^{T}\frac{u_t}{2^t}: \quad u_1,\cdots,u_{T} \in \{\pm 1\}\right\}. 
\end{align*}
It is easy to verify the following lexicographical order holds for the value of $f$:  
\begin{align*}
	f(u_1,\cdots,u_{t-1},1,u_{t+1},\cdots,u_{T}) > 	f(u_1,\cdots,u_{t-1},-1,u_{t+1}',\cdots,u_{T}'). 
\end{align*}
Conditioned on the realization of $v^\star = f(u_1,\cdots,u_{T})$, define the sequence $(v_t)_{t\in[T]}$ as follows: $v_1 = 1/2$, and $v_{t+1} = v_t + u_t/2^{t+2}$ for all $t\ge 1$. This defines a valid joint distribution $\nu$ on $(v^\star, (v_t)_{t\in [T]})$. 

We first understand the oracle performance. Since both $P$ and $Q$ are singletons, and $v_t\in [1/4,3/4]$ almost surely, the optimal bidding price is simply $b_t^\star = m_t = 1/8\cdot \1(v_t \ge v^\star)$. The expected revenue achieved by the oracle is then
\begin{align*}
	\bE_{\nu}\left[\sum_{t=1}^T (v_t - m_t) \right] = 	\bE_{\nu}\left[\sum_{t=1}^T \left(v_t - \frac{1}{8}\cdot \1(v_t\ge v^\star)\right) \right]. 
\end{align*}

Next we understand the bidder's performance. Here is the central claim in the analysis: conditioned on $(v_1,\cdots,v_t,m_1,\cdots,m_{t-1})$, the distribution of $m_t$ is uniform over the two points $\{0, 1/8\}$ under $\nu$. This is due to $m_t = 1/8\cdot \1(v_t \ge v^\star) = 1/8\cdot \1(u_t=-1)$, where the last identity follows from
\begin{align*}
	\text{sign}(v_t - v^\star) = \text{sign}\left(\frac{1}{2} + \frac{1}{4}\sum_{s=1}^{t-1} \frac{u_s}{2^s} - \left(\frac{1}{2} + \frac{1}{4}\sum_{s=1}^T \frac{u_s}{2^s} \right)\right) = -\text{sign}(u_t). 
\end{align*}
Consequently, the history $(v_1,\cdots,v_t,m_1,\cdots,m_{t-1})$ is a deterministic function of $(u_1,\cdots,u_{t-1})$, while $m_t$ is determined only by $u_t$. As we assume a uniform distribution over $(u_1,\cdots,u_T)\in \{\pm 1\}^T$, we conclude that $m_t$ is uniformly distributed on $\{0, 1/8\}$ under $\nu$, conditioned on all observed history. Hence, when the bidder bids a price $b_t$ at time $t$, the expected instantaneous reward with respect to the posterior distribution is (note that $v_t\ge 1/4$ almost surely)
\begin{align*}
	\bE_{m_t\sim \text{Uniform}(\{0, 1/8\}) }[(v_t - b_t)\1(b_t \ge m_t)] &\le \max_{b\in [0,1]}  \frac{v_t - b}{2}\left(1 + \1\left(b\ge \frac{1}{8}\right)\right) \\
	&= \max\left\{\frac{v_t}{2}, v_t - \frac{1}{8} \right\} = v_t - \frac{1}{8}. 
\end{align*}
Consequently, the expected revenue achieved by any bidder is at most $\bE_\nu[\sum_{t=1}^T(v_t-1/8)]$, and the average regret with respect to $\nu$ is at least 
\begin{align*}
	\bE_\nu\left[\sum_{t=1}^T \left(\frac{1}{8} - m_t\right)\right] = 	\bE_\nu\left[\sum_{t=1}^T \left(\frac{1}{8} - \frac{1}{8}\cdot \1(u_t=-1)\right)\right] = \frac{T}{16}, 
\end{align*}
as desired. 
\end{proof}

The negative result in Theorem \ref{thm:monotone} suggests to compare with a weaker oracle, as opposed to the current strong oracle who can bid the best price at every round. This is an interesting direction we defer for future study.


\subsection{Concluding Remarks and Open Problems}
In this paper, we formulate the bidding problem of a single bidder in repeated first-price auctions as a contextual bandit problem, and devise two algorithms which achieve the near-optimal $\widetilde{O}(\sqrt{T})$ regret against a powerful oracle, when the others' bids are iid and the transaction price is revealed. In particular, we achieve the near-optimal learning performance in first-price auctions via both the development of general bandit algorithms for partially ordered contextual bandits, and an in-depth exploitation of the correlated reward structures in first-price auctions. 

We also point out some open directions. First, it remains unknown as to whether the $\Omega(T)$ lower bound still holds when $m_t$ is adversarial but $v_t$ is stochastic. It will be interesting to see if learning turns out to be possible, and if the full-information feedback and the winner-only feedback make a difference. Second, the stochastic $m_t$ setting and the adversarial $m_t$ setting lie at the two ends of the spectrum, where practice often lies somewhere in between. For example, $m_t$ could be a consequence of stragetic behaviors, and bidders could also leverage other information not modeled in this paper to predict $m_t$. Could we have a more refined modeling of $m_t$ (perhaps via some predictive modelling as a blackbox) lead to better results, in particular, to results that would depend on how accurate $m_t$ can be predicted? A positive result of this type would empower the bidding platform with all the existing blackbox supervised learning tools that are readily available off the shelf. Finally, we have assumed that the private values are known as they are independent of the bidding process. In practical online ads, private values represent the treatment effects of given advertising slots and are not entirely independent of the bidding process. A joint value estimation and bidding (such as \cite{waisman2019online}) is an interesting future direction. 

%% file: appendix.tex
\section{Auxiliary Lemmas}
\begin{lemma}\cite[Lemma 8]{alon2015online}\label{lemma.independence_number}
Let $G$ be a directed graph with $n$ vertices and independence number $\alpha$. Then there exists a subset $V$ of vertices with $|V|\le 50\alpha \log n$ such that every vertex of $G$ is a child of some vertex in $V$. 
\end{lemma} 
\begin{remark}
Such a set $V$ is called a \emph{weakly dominating set}, and the smallest cardinality of all such $V$ is called the \emph{weak domination number} of $G$, usually denoted by $\delta(G)$. Lemma \ref{lemma.independence_number} then shows the following inequality for general directed graph $G$: $\delta(G)\le 50\alpha(G)\log n$. 
\end{remark}

\begin{lemma}\cite[Corollary 2.2]{de2004self}\label{lemma.self-normalized_martingale}
	If two random variables $A,B$ satisfy $\bE[\exp(\lambda A - \lambda^2B^2/2)]\le 1$ for any $\lambda\in \bR$, then for any $x\ge \sqrt{2}$ and $y>0$ we have
	\begin{align*}
	\bP\left(|A| \bigg/ \sqrt{(B^2 + y)\left(1+ \frac{1}{2}\log\left(1+\frac{B^2}{y}\right) \right)} \ge x \right) \le \exp\left(-\frac{x^2}{2}\right). 
	\end{align*}
\end{lemma}

\begin{lemma}\label{lemma.TV_KL}
	Let $P,Q$ be two probability measures on the same probability space. Then 
	\begin{align*}
	1 - \|P-Q\|_{\text{\rm TV}} \ge \frac{1}{2}\exp\left( - \frac{D_{\text{\rm KL}}(P \| Q) +  D_{\text{\rm KL}}(Q \| P)}{2}\right).
	\end{align*}
\end{lemma}
\begin{proof}
	The lemma follows from \cite[Lemma 2.6]{Tsybakov2008} and the convexity of $x\mapsto \exp(-x)$. 
\end{proof}
	
\begin{lemma}\cite[Lemma 3 - special case]{gao2019batched}\label{lemma.tree}
Let $Q_1,\cdots, Q_n$ be probability measures on some common probability space $(\Omega, \calF)$, and $\Psi: \Omega\to [n]$ be any measurable function (i.e. test). Then
\begin{align*}
	\frac{1}{n}\sum_{i=1}^n Q_i(\Psi \neq i) \ge \frac{1}{2n}\sum_{i=2}^n \exp\left( - D_{\text{\rm KL}}(Q_1 \| Q_i)\right). 
\end{align*}
\end{lemma}

\begin{lemma}[Bernstein inequality \cite{Boucheron--Lugosi--Massart2013}]\label{lemma.bennett}
	Let $X_1,\ldots,X_n\in [a,b]$ be independent random variables with 
	\begin{align*}
	\sigma^2 \triangleq \sum_{i=1}^n \var(X_i).
	\end{align*}
	Then we have
	\begin{align*}
	\bP\left(\left|\sum_{i=1}^n X_i - \sum_{i=1}^n \bE[X_i]\right|\ge \varepsilon\right)\le 2\exp\left(-\frac{\varepsilon^2}{2(\sigma^2+(b-a)\varepsilon/3)}\right).
	\end{align*}
\end{lemma}

The next lemma states the Chernoff bound for Binomial random variables; this is an immediate consequence of Lemma \ref{lemma.bennett}. 
\begin{lemma}[\!\!{\cite[Theorem 5.4]{mitzenmacher2017probability}}]\label{lemma:binomial_tail}
	For all $n\in \mathbb{N}$, $p \in [0,1]$, and $x\geq0$ we have
	\begin{align*}
		\bP(\mathsf{B}(n,p)-np \geq x) &\leq \exp\left(-\frac{x^2}{2(np+x)}\right), \\
		\bP(\mathsf{B}(n,p)-np \leq -x) &\leq \exp\left(-\frac{x^2}{2np}\right).
	\end{align*}
\end{lemma}

\section{Proof of the Regret Lower Bound}\label{appendix.lower_bound}
This section is devoted to the proof of the following $\Omega(\sqrt{T})$ lower bound on the regret. 
\begin{theorem}\label{thm:sqrtT_lower_bound}
	Even in the special case where $v_t\equiv 1$ and HOBs $m_t$ are always revealed at the end of each round, there exists an absolute constant $c>0$ independent of $T$ such that
	\begin{align*}
	\inf_{\pi}\sup_{G}\bE_G[R_T(\pi;v)] \ge c\sqrt{T}, 
	\end{align*}
	where the supremum is taken over all possible CDFs $G(\cdot)$ of the HOB $m_t$, and the infimum is taken over all possible policies $\pi$. 
\end{theorem}

The proof of Theorem \ref{thm:sqrtT_lower_bound} is the usual manifestation of the Le Cam's two-point method \cite{Tsybakov2008}. Consider the following two candidates of the CDFs supported on $[0,1]$: 
\begin{align*}
G_1(x) = \begin{cases}
0 & \text{if } 0\le x< \frac{1}{3} \\
\frac{1}{2} + \Delta & \text{if } \frac{1}{3}\le x < \frac{2}{3} \\
1 & \text{if } x \ge \frac{2}{3}
\end{cases}, \qquad 
G_2(x) = \begin{cases}
0 & \text{if } 0\le x< \frac{1}{3} \\
\frac{1}{2} - \Delta & \text{if } \frac{1}{3}\le x < \frac{2}{3} \\
1 & \text{if } x \ge \frac{2}{3}
\end{cases},
\end{align*}
where $\Delta\in (0,1/4)$ is some parameter to be chosen later. In other words, the CDF $G_1$ corresponds to a discrete random variable taking value in $\{1/3, 2/3\}$ with probability $(1/2 + \Delta, 1/2 - \Delta)$, and the CDF $G_2$ corresponds to the probability $(1/2 - \Delta, 1/2 + \Delta)$. Let $R_1(v_t, b_t)$ and $R_2(v_t,b_t)$ be the expected reward in \eqref{eq.expected_reward} averaged under the CDF $G_1$ and $G_2$, respectively. After some algebra, it is straightforward to check that
\begin{align*}
\max_{b\in [0,1]} R_1(v_t, b) &= \max_{b\in [0,1]} (1-b)G_1(b) = \frac{1+2\Delta}{3}, \\
\max_{b\in [0,1]} R_2(v_t, b) &= \max_{b\in [0,1]} (1-b)G_2(b) = \frac{1}{3}, \\
\max_{b\in [0,1]} (R_1(v_t,b) + R_2(v_t,b)) &= \max_{b\in [0,1]} (1-b)(G_1(b) + G_2(b)) = \frac{2}{3}. 
\end{align*}
Hence, for any $b_t\in [0,1]$, we have
\begin{align}\label{eq:separation}
&\left(\max_{b\in [0,1]} R_1(v_t, b)  - R_1(v_t,b_t) \right) + \left( \max_{b\in [0,1]} R_2(v_t, b) - R_2(v_t,b_t) \right) \nonumber\\
&\ge \max_{b\in [0,1]} R_1(v_t, b) + \max_{b\in [0,1]} R_2(v_t, b)  - \max_{b\in [0,1]} (R_1(v_t,b) + R_2(v_t,b)) \nonumber\\
&= \frac{2\Delta}{3}. 
\end{align}
The inequality \eqref{eq:separation} is the separation condition required in the two-point method: there is no single bid $b_t$ which gives a uniformly small instantaneous regret under both CDFs $G_1$ and $G_2$. 

For $i\in \{1,2\}$, let $P_i^{\otimes (t-1)}$ be the distribution of all observables $(m_1,\cdots,m_{t-1})$ at the beginning of time $t$. Then for any policy $\pi$, 
\begin{align}\label{eq:minimax_regret}
\sup_{G}\bE_G[R_T(\pi;v)] &\stepa{\ge} \frac{1}{2} \bE_{G_1} [R_T(\pi;v)] + \frac{1}{2} \bE_{G_2} [R_T(\pi;v)] \nonumber\\
&= \frac{1}{2}\sum_{t=1}^T \left( \bE_{P_1^{\otimes (t-1)}}\left[\max_{b\in [0,1]} R_1(v_t,b) - R_1(v_t,b_t)\right] +  \bE_{P_2^{\otimes (t-1)}}\left[\max_{b\in [0,1]} R_2(v_t,b) - R_2(v_t,b_t)\right] \right) \nonumber\\
&\stepb{\ge} \frac{1}{2}\sum_{t=1}^T \frac{2\Delta}{3}\int \min\{\mathrm{d}P_1^{\otimes (t-1)}, \mathrm{d}P_2^{\otimes (t-1)}\} \nonumber\\
&\stepc{=} \frac{1}{2}\sum_{t=1}^T \frac{2\Delta}{3}(1 - \|P_1^{\otimes (t-1)} - P_2^{\otimes (t-1)}\|_{\text{TV}}) \nonumber\\
&\stepd{\ge } \frac{\Delta T}{3}\left(1 - \|P_1^{\otimes T} - P_2^{\otimes T}\|_{\text{TV}} \right),
\end{align}
where (a) is due to the fact that the maximum is no smaller than the average, (b) follows from \eqref{eq:separation}, (c) is due to the identity $\int\min\{\mathrm{d}P,\mathrm{d}Q\} = 1 - \|P-Q\|_{\text{TV}}$, and (d) is due to the data-processing inequality $\|P_1^{\otimes (t-1)} - P_2^{\otimes (t-1)}\|_{\text{TV}} \le \|P_1^{\otimes T}- P_2^{\otimes T}\|_{\text{TV}}$ for the total variation distance. Invoking Lemma \ref{lemma.TV_KL} and using the fact that for $\Delta\in (0,1/4)$, 
\begin{align*}
D_{\text{KL}}(P_1^{\otimes T} \| P_2^{\otimes T}) &= TD_{\text{KL}}(G_1 \| G_2) \\
&= T\left(\left(\frac{1}{2}-\Delta\right)\log\frac{1/2 - \Delta}{1/2 + \Delta} + \left(\frac{1}{2}+\Delta\right)\log\frac{1/2 + \Delta}{1/2 - \Delta}\right) \\
&\le 32T\Delta^2, 
\end{align*}
we have the following inequality on the total variation distance: 
\begin{align}\label{eq:TV_lower_bound}
1 - \|P_1^{\otimes T} - P_2^{\otimes T}\|_{\text{TV}}  \ge \frac{1}{2}\exp\left( - 32T\Delta^2\right). 
\end{align}
Finally, choosing $\Delta = 1/(4\sqrt{T})$ and combining \eqref{eq:minimax_regret}, \eqref{eq:TV_lower_bound}, we conclude that Theorem \ref{thm:sqrtT_lower_bound} holds with the constant $c = 1/(24e^2)$. 

\section{Proof of Main Lemmas}
\subsection{Proof of Lemma \ref{lemma.concentration}}
Throughout the proof we use $a \to b$ to denote that either $a=b$, or $(a,b)\in E$ in the feedback graph. By definition, we have $n_{a}^t = \sum_{s=1}^t \1(a_s \to a)$. Since the statement of the lemma is vacuous when $n_a^t = 0$, in the sequel we assume that $n_a^t \ge 1$. Consequently,
\begin{align*}
	(n_{a}^t)^{\frac{1}{2}}|\bar{r}_{a}^t - R_{a}| 
	&= \frac{\left| \sum_{s=1}^t \1(a_s\to a)( r_{s,a} - R_{a} )\right|}{\sqrt{\sum_{s=1}^t \1(a_s\to a)}}.  
\end{align*}
Hence, by a union bound, the target probability is at least $1-\delta$ if the following pointwise inequality holds: for fixed $a\in\calA, t\in [T]$, it holds that
\begin{align}\label{eq.pointwise}
	\bP\left(\frac{\left| \sum_{s=1}^t \1(a_s\to a)( r_{s,a} - R_{a} )\right|}{\sqrt{\sum_{s=1}^t \1(a_s\to a)}} > \sqrt{9\log(2T)\log\left(\frac{KT}{\delta}\right)} \right) \le \frac{\delta}{KT}. 
\end{align}

One technical difficulty in proving \eqref{eq.pointwise} is that the choice of $a_t$ may depend on the previous observations $\{r_{s,a}\}_{s<t}$, and therefore the summands on the numerator are not mutually independent. However, if we define $\calF_t = \sigma( \{a_s, (r_{s,a})_{a\in\calA} \}_{s\le t} )$ to be the $\sigma$-field of the historic observations up to time $t$, then $a_t$ is $\calF_{t-1}$-measurable and $r_{t,a}$ is independent of $\calF_{t-1}$. Consequently 
\begin{align*}
	M_t := \sum_{s=1}^t \1(a_s\to a)( r_{s,a} - R_{a} )
\end{align*}
is a martingale adapted to the filtration $\{\calF_t\}_{t\ge 1}$, where $M_0 = 0$. Further define the predictable quadratic variation and the total quadratic variation of $M_t$ respectively as
\begin{align*}
	\langle M\rangle_t := \sum_{s=1}^t \bE[(M_s - M_{s-1})^2 | \calF_{s-1}], \qquad [M]_t := \sum_{s=1}^t (M_s - M_{s-1})^2, 
\end{align*}
then the theory of self-normalized martingales \cite[Lemma B.1]{bercu2008exponential} gives that $\bE[\exp(\lambda M_t - \lambda^2(\langle M\rangle_t + [M]_t)/2)]\le 1$ holds for all $\lambda\in \bR$, i.e., the choices $A=M_t$ and $B^2=\langle M\rangle_t+ [M]_t$ fulfill the condition of Lemma \ref{lemma.self-normalized_martingale}. Thanks to the boundedness assumption $r_{s,c,a}\in [0,1]$, simple algebra gives $\langle M\rangle_t+ [M]_t \le 2\sum_{s\le t} \1(a_s\to a)$, and the choice $y=1$ leads to (recall that $\sum_{s=1}^t \1(a_s\to a)\ge 1$)
\begin{align*}
	(B^2 + y)\left(1+ \frac{1}{2}\log\left(1+\frac{B^2}{y}\right) \right) &\le 3\sum_{s=1}^t \1(a_s\to a)\cdot \left(1+\frac{1}{2}\log(1+2T)\right) \\
	& \le 9\log(2T)\cdot \sum_{s=1}^t \1(a_s\to a). 
\end{align*}
Hence, applying Lemma \ref{lemma.self-normalized_martingale} with $y=1$ yields an upper bound $\delta/KT$ of the deviation probability in \eqref{eq.pointwise}. 

\subsection{Proof of Lemma \ref{lemma.partial_order_N_lb}}
Throughout the proof we fix the randomness in $(c_t)_{t\in [T]}$ and assume that $(c_t)_{t\in [T]}$ is a fixed sequence. We prove the following statement: for a given $c\in \calC$, the probability that the quantity $N_c^t$ remains the same during $t\in [t_1, t_2]$ with $\sum_{t_1<s\le t_2} \1(c_s\preceq_\calC c) \ge m := \alpha\log(MKT^2/\delta)$ is at most $\delta/(MT^2)$. In other words, whenever $m$ contexts smaller than $c$ have appeared in the history, the quantity $N_c^t$ should at least increase by one with probability at least $1-\delta/(MT^2)$. By a union bound over $c\in \calC$ and $t_1, t_2 \in [T]$, we conclude that with high probability at least $1-\delta$, it holds that
\begin{align*}
	N_c^t \ge \bigg\lfloor \frac{1}{m}\sum_{s\le t} \1(c_s\preceq_\calC c) \bigg\rfloor
\end{align*}
for every $c\in \calC$ and $t\in [T]$. This implies the desired inequality. 

To prove the above statement, let $\calE$ be the event that the quantity $N_c^t$ remains the same during $t\in [t_1, t_2]$ with $\sum_{t_1<s\le t_2} \1(c_s\preceq_\calC c) \ge m$, and $\calE_a$ be the event that the reward of action $a\in \calA$ (under any context, thanks to the cross learning structure) is never revealed during $t\in [t_1, t_2]$. We claim that 
\begin{align*}
	\calE\subseteq \cup_{a\in \calA_{\text{act}, c}^{t_2}} \calE_a. 
\end{align*}
To see this, if the RHS is not true, then the count $n_{a}^t$ for every $a\in \calA_{\text{act},c}^{t_2}$ has been increased by at least one during $[t_1, t_2]$, and thus $N_{c}^t = \min_{a\in \calA_{\text{act},c}^t} n_a^t$ must have been increased by at least one during $[t_1, t_2]$, i.e. $\calE$ does not hold. Consequently, by the union bound, it suffices to prove that $\bP(\calE_a) \le \delta/(MKT^2)$ for any given $a\in \calA_{\text{act}, c}^{t_2}$. 

Now consider any time point $s\in (t_1, t_2]$ such that $c_s\preceq_\calC c$. By the additional arm elimination rule in Algorithm \ref{algo.partial_order}, every action $a\in \calA_{\text{act},c}^{t_2}\subseteq \calA_{\text{act},c}^{s-1}$ must be dominated by some action $a'\in \calA_{\text{act}, c_s}^{s-1}$, and therefore by some further action $a'' \in \calA_{\text{act}, c_s}^{s-1}$ with no parent due to the partial order structure $(\calA, \preceq_G)$ in Definition \ref{assump:partial_order}. In other words, let $\calB$ be the set of actions in $\calA_{\text{act}, c_s}^{s-1}$ with no parent, there is at least one action in $\calB$ that reveals the reward of action $a$. On the other hand, since $\calB$ must be an independent set by the proof of Theorem \ref{thm.MAB_feedback}, we have $|\calB| \le \alpha$. According to Algorithm \ref{algo.partial_order}, the learner chooses an action in $\calB$ uniformly at random, so the probability that the reward of action $a$ is not revealed at time $s$ is at most $1-1/\alpha$. Since the random choices are independent across time, we have
\begin{align*}
	\bP(\calE_a) \le \left(1 - \frac{1}{\alpha}\right)^m \le e^{-m/\alpha} = \frac{\delta}{MKT^2}, 
\end{align*}
which is what we want. 

\subsection{Proof of Lemma \ref{lemma.record_breaking}}
We first prove the first statement for the case $\beta = 1$, where $(\calX, \preceq)$ is a chain. Let $x_1 \preceq x_2 \preceq \cdots \preceq x_m$ be all elements of $\calX$, and $F(x) := \bP(X_1 \preceq x)$ be the common CDF of $X_1, \cdots, X_T$. Conditioned on $X_t$, the sum $\sum_{s<t} \1(X_s\le X_t)$ follows a Binomial distribution $\mathsf{B}(t-1, F(X_t))$. Since
\begin{align*}
	\bE\left[\frac{1}{1+\mathsf{B}(n,p)}\right] &= \sum_{k=0}^n \frac{1}{1+k}\binom{n}{k}p^k(1-p)^{n-k} 
	= \sum_{k=0}^n \frac{1}{(n+1)p}\binom{n+1}{k+1}p^{k+1}(1-p)^{n-k} \\
	&=  \frac{1}{(n+1)p} \sum_{i=1}^{n+1}\binom{n+1}{i}p^{i}(1-p)^{n+1-i} = \frac{1}{(n+1)p}\left((p+1-p)^{n+1} - (1-p)^{n+1}\right) \\
	&=  \frac{1}{(n+1)p}\left(1 - (1-p)^{n+1}\right) \le \min\left\{\frac{1}{(n+1)p}, 1 \right\},
\end{align*}
we have $\bE\left[\left(1 + \sum_{s<t} \1(X_s \preceq X_t) \right)^{-1} \right] \le \bE\left[\min\left\{\frac{1}{tF(X_t)},1\right\}\right]. $
Further, we have $\bP(F(X_t)\le u) \le u$ for any $u\in [0,1)$. To see this, partition $[0,1)$ into $[0, F(x_1)), [F(x_1), F(x_2)), \cdots, [F(x_{m-1}), F(x_m))$, then
\begin{align*}
	\bP(F(X_t)\le u) = 0\cdot \1(u\in [0,F(x_1))) + \sum_{i=1}^{m-1} F(x_i)\cdot \1(u\in [F(x_{i}), F(x_{i+1})) ) \le u.
\end{align*}
Consequently, $F(X_t)$ stochastically dominates $U\sim \mathsf{Unif}([0,1])$. Since $x\in [0,1]\mapsto \min\{1/(tx), 1 \}$ is decreasing, we have:
\begin{align*}
	\bE\left[\min\left\{\frac{1}{tF(X_t)},1\right\}\right] \le \bE\left[\min\left\{\frac{1}{tU},1\right\}\right] = \frac{1 + \log t}{t}. 
\end{align*}
The proof is completed by noting that
$\sum_{t=1}^T \frac{1 + \log t}{t} \le (1+\log T)\cdot \sum_{t=1}^{T} \frac{1}{t} \le (1+\log T)^2$. 

For general $\beta$, Dilworth's theorem \cite{dilworth1950decomposition} states that $(\calX, \preceq)$ could be decomposed into $\beta$ disjoint chains, i.e. there is a partition of $\calX = \cup_{i=1}^\beta \calX_i$ such that each $(\calX_i, \preceq)$ is a chain. Let $\Phi(x) \in [\beta]$ be the membership of $x\in \calX$, i.e. $\Phi(x) = i$ if and only if $x\in \calX_i$. Then
\begin{align*}
	\sum_{t=1}^T \frac{1}{1+\sum_{s<t} \1(X_s\preceq X_t)} &= \sum_{i=1}^\beta \sum_{t: \Phi(X_t) = i} \frac{1}{1+\sum_{s<t} \1(X_s\preceq X_t)} \\
	& \le \sum_{i=1}^\beta \sum_{t: \Phi(X_t) = i} \frac{1}{1+\sum_{s<t: \Phi(X_s)=i} \1(X_s\preceq X_t)}. 
\end{align*}
Conditioned on any realization of $(\Phi(X_t))_{t\in [T]}$, the random variables $(X_t)_{t\in [T]: \Phi(X_t)=i}$ are still conditional i.i.d., therefore the expectation of each inner sum is at most $(1+\log T)^2$, as desired. 

For the second statement, the following chain of inequality leads to the claimed result: 
\begin{align*}
	\sum_{t=1}^T \frac{1}{1+\sum_{s<t}\1(x_s\preceq x_t)} &\le \sum_{t=1}^T \frac{1}{1+\sum_{s<t}\1(x_s = x_t)} \\
	&= \sum_{x\in \calX}\sum_{t=1}^T \frac{\1(x_t = x)}{1+\sum_{s<t}\1(x_s = x_t)} \\
	&= \sum_{x\in \calX}\sum_{t: x_t = x} \frac{1}{1+\sum_{s<t}\1(x_s = x)} \\
	&\le \sum_{x\in \calX} \left(1 + \frac{1}{2}  + \cdots + \frac{1}{T}\right) \\
	&\le |\calX|(1+\log T). 
\end{align*}

\subsection{Proof of Lemma \ref{lemma.approx_regret}}
By the definition of regret in a first-price auction, we have:	
\begin{align*}
R_T(\pi,v) &= \sum_{t=1}^T \left(\max_{b\in [0,1]} R(v_t, b) - R(v_t,b_t) \right) \\
& =
\sum_{t=1}^T \left(\max_{b\in [0,1]} R(v_t, b) - \max_{b\in \calA} R(\widetilde{v}_t, b) + \max_{b\in \calA} R(\widetilde{v}_t, b)- R(\widetilde{v}_t,b_t) + R(\widetilde{v}_t,b_t) -R(v_t,b_t) \right) \\
& = \widetilde{R}_T(\pi^{\text{\rm Q}},v) +  \sum_{t=1}^T \left(R(\widetilde{v}_t,b_t) -R(v_t,b_t) \right)+ \sum_{t=1}^T \left(\max_{b\in [0,1]} R(v_t, b) - \max_{b\in \calA} R(\widetilde{v}_t, b)\right) \\
&\le  \widetilde{R}_T(\pi^{\text{\rm Q}},v) + \frac{T}{M} + \left(\frac{T}{M} + \frac{T}{K}\right) \\
&= \widetilde{R}_T(\pi^{\text{\rm Q}},v) + \frac{2T}{M} + \frac{T}{K},
\end{align*}
where the inequality follows from 
\begin{align*}
|R(v,b) - R(\widetilde{v},b)| = |v - \widetilde{v}|G(b) \le  |v - \widetilde{v}| \le \frac{1}{M}, \qquad \forall v,b\in [0,1], 
\end{align*}
and for any $v\in [0,1]$ with $\widetilde{b} = \min\{b'\in \calA: b'\ge b\}$, it holds that
\begin{align*}
\max_{b\in [0,1]} R(v,b) = \max_{0\le b\le v} R(v,b) \le \max_{0\le b\le v} (v-b)G(\widetilde{b}) \le \max_{0\le b\le v} (v-\widetilde{b})G(\widetilde{b}) + \frac{1}{K} \le \max_{b\in \calA} R(v,b) + \frac{1}{K}. 
\end{align*}

\subsection{Proof of Lemma \ref{lemma.monotonicity}}
Fix any $v_1,v_2\in [0,1]$ with $v_1\le v_2$. Then for any $b\le b^\star(v_1)$, it holds that
\begin{align*}
R(v_2, b) &= R(v_1,b) + (v_2 - v_1)G(b) \\
&\le R(v_1, b^\star(v_1)) + (v_2 - v_1)G(b^\star(v_1)) \\
&= R(v_2, b^\star(v_1)), 
\end{align*}
where the inequality follows from the definition of $b^\star(v_1)$ and the assumption $b\le b^\star(v_1)$. Hence, all bids $b\le b^\star(v_1)$ cannot be the largest maximizer of $R(v_2,b)$, and $b^\star(v_1)\le b^\star(v_2)$ as claimed.

\subsection{Proof of Lemma \ref{lemma.UCB}}
We first prove the following inequality: for each $i\in [K], \ell\in [L], T_0<t\le T$ and $\gamma\ge 3$, 
\begin{align}\label{eq:target}
\bP\left(\left|\widehat{r}_{i,t}^{\ell} - r_{i,t} \right| \ge \gamma\left(\sqrt{\log(LKT)\sum_{j\ge i}\frac{p_j}{n_{j,t}^{\ell}} }+ \frac{\log(LKT)}{n_{i,t}^{\ell}} \right) \right) \le \frac{2}{(LKT)^4}. 
\end{align}
In fact, by the definition of $\widehat{r}_{i,t}^{\ell}$ in \eqref{eq:reward_level}, we have
\begin{align*}
\left|\widehat{r}_{i,t}^{\ell} - r_{i,t} \right| &= |v_t - b^i|\cdot \left|\sum_{j\ge i}  \frac{\sum_{s\in \Phi^{\ell}(t-1)} \1(b_s\le b^j)(\1(b^j < m_s\le b^{j+1}) - p_j)}{\sum_{s\in \Phi^{\ell}(t-1)} \1(b_s\le b^j)} \right| \\
&\le \left|\sum_{j\ge i}  \frac{\sum_{s\in \Phi^{\ell}(t-1)} \1(b_s\le b^j)(\1(b^j < m_s\le b^{j+1}) - p_j)}{\sum_{s\in \Phi^{\ell}(t-1)} \1(b_s\le b^j)} \right| \\
&= \left| \sum_{s\in \Phi^{\ell}(t-1)} \sum_{j\ge i} \frac{\1(b_s\le b^j)}{n_{j,t}^{\ell}}(\1(b^j < m_s\le b^{j+1}) - p_j) \right|. 
\end{align*}
Now conditioning on the variables $(b_s: s\in \Phi^{\ell}(t-1))$, the only random variables appearing in the target inequality \eqref{eq:target} are $(m_s: s\in \Phi^{\ell}(t-1))$, which by Lemma \ref{lemma:independence} are still mutually independent. We will show that \eqref{eq:target} holds even after conditioning on $(b_s: s\in \Phi^{\ell}(t-1))$, and therefore the unconditional version holds as well. Hence, in the sequel we assume that the conditioning is always performed. 

For $s\in \Phi^{\ell}(t-1)$, define a random variable
\begin{align}\label{eq:X_j}
X_s \triangleq \sum_{j\ge i} \frac{\1(b_s\le b^j)}{n_{j,t}^{\ell}}(\1(b^j < m_s\le b^{j+1}) - p_j). 
\end{align}
Clearly $\bE[X_s] = 0$. Moreover, as the map $j\mapsto n_{j,t}^{\ell}=\sum_{s\in \Phi^{\ell}(t-1)} \1(b_s\le b^j)$ is always non-decreasing, the range of each $X_s$, i.e. the difference of the largest and smallest possible values of $X_s$, is at most $1/n_{i,t}^{\ell}$. Furthermore, 
\begin{align}\label{eq:variance_individual}
\bE[X_s^2] &= \sum_{j\ge i} \frac{\1(b_s\le b^j)}{(n_{j,t}^{\ell})^2}\cdot \bE[(\1(b^j < m_s\le b^{j+1}) - p_j)^2] \nonumber\\
&\quad + \sum_{j\ge i, j'\ge i, j\neq j'}\frac{\1(b_s\le b^j)}{n_{j,t}^{\ell}}\frac{\1(b_s\le b^{j'})}{n_{j',t}^{\ell}}\cdot \bE[(\1(b^j < m_s\le b^{j+1}) - p_j)(\1(b^{j'} < m_s\le b^{j'+1}) - p_{j'})]\nonumber \\
&= \sum_{j\ge i} \frac{\1(b_s\le b^j)}{(n_{j,t}^{\ell})^2}\cdot p_j(1-p_j) - \sum_{j\ge i, j'\ge i, j\neq j'}\frac{\1(b_s\le b^j)}{n_{j,t}^{\ell}}\frac{\1(b_s\le b^{j'})}{n_{j',t}^{\ell}}\cdot p_jp_{j'} \nonumber\\
&\le \sum_{j\ge i} \frac{\1(b_s\le b^j)}{(n_{j,t}^{\ell})^2} p_j. 
\end{align}
As the random variables $(X_s: s\in \Phi^{\ell}(t-1))$ are mutually independent after the conditioning, the individual variance in \eqref{eq:variance_individual} leads to
\begin{align}\label{eq:variance_total}
\var\left(\sum_{s\in \Phi^{\ell}(t-1)} X_s\right) = \sum_{s\in \Phi^{\ell}(t-1)} \bE[X_s^2] \le \sum_{s\in \Phi^{\ell}(t-1)}\sum_{j\ge i} \frac{\1(b_s\le b^j)}{(n_{j,t}^{\ell})^2} p_j = \sum_{j\ge i} \frac{p_j}{n_{j,t}^{\ell}}. 
\end{align}
Hence, by the Bernstein inequality (cf. Lemma \ref{lemma.bennett}) applied to the variance upper bound in \eqref{eq:variance_total} and the range upper bound of each $X_s$, we conclude that
\begin{align*}
\bP\left(\left|\sum_{s\in \Phi^{\ell}(t-1)} X_s \right| \ge \gamma\left(\sqrt{ \log(LKT)\sum_{j\ge i} \frac{p_j}{n_{j,t}^{\ell}}}  + \frac{\log(LKT)}{n_{i,t}^{\ell}}\right) \right) \le 2(LKT)^{-\min\{\gamma^2,3\gamma\}/2}. 
\end{align*}
Choosing $\gamma\ge 3$ completes the proof of the target inequality \eqref{eq:target}. 

Now we use \eqref{eq:target} to prove Lemma \ref{lemma.UCB}. First note that by a union bound over $i\in [K], \ell\in [L]$ and $T_0\le t\le T$, the following inequality holds uniformly over all $(i,\ell,t)$ tuples with probability at least $1-2T^{-3}$: 
\begin{align}\label{eq:ucb_true}
\left|\widehat{r}_{i,t}^{\ell} - r_{i,t} \right| \le \gamma\left(\sqrt{\log(LKT)\sum_{j\ge i}\frac{p_j}{n_{j,t}^{\ell}} }+ \frac{\log(LKT)}{n_{i,t}^{\ell}} \right). 
\end{align}
To connect the RHS of \eqref{eq:ucb_true} with the width in \eqref{eq:width}, we need to relate the true probability $p_j$ with its estimate $\widehat{p}_j^0$ using $m:=\lceil \sqrt{T} \rceil$ held-out samples. As $m\widehat{p}_j^0\sim \mathsf{B}(m,p_j)$, Lemma \ref{lemma:binomial_tail} shows that with probability at least $1-(KT)^{-2}$, the following inequalities hold: 
\begin{align}
	p_j &\le 2\left(\widehat{p}_j^0 + \frac{\gamma \log(KT)}{m}\right), \label{eq:p_0_lower} \\
\widehat{p}_j^0 &\le 2\left(p_j + \frac{\gamma \log(KT)}{m}\right). \label{eq:p_0_upper}  
\end{align}
To be more specific, we can establish \eqref{eq:p_0_lower} and \eqref{eq:p_0_upper} by distinguishing into the cases $p_j \le \gamma\log(KT)/m$ and $p_j > \gamma\log(KT)/m$, and apply Lemma \ref{lemma:binomial_tail} in the respective scenarios. For example, there is nothing to prove for \eqref{eq:p_0_lower} if $p_j\le \gamma\log(KT)/m$, while Lemma \ref{lemma:binomial_tail} shows that $\widehat{p}_j^0 \ge p_j/2$ holds with high probability for $p_j>\gamma\log(KT)/m$. By a union bound over $j\in [K]$, the inequalities  \eqref{eq:p_0_lower} and \eqref{eq:p_0_upper} hold for all $j\in [K]$ with probability at least $1-(TK)^{-1}$. 

Consequently, based on \eqref{eq:ucb_true}, it holds that
\begin{align*}
	\left|\widehat{r}_{i,t}^{\ell} - r_{i,t} \right| &\le \gamma\left(\sqrt{\log(LKT)\sum_{j\ge i}\frac{p_j}{n_{j,t}^{\ell}} }+ \frac{\log(LKT)}{n_{i,t}^{\ell}} \right)  \\
	&\stepa{\le} \gamma\left(\sqrt{2\log(LKT)\sum_{j\ge i}\frac{1}{n_{j,t}^{\ell}} \left(\widehat{p}_j^0 + \frac{\gamma \log(LKT)}{\lceil \sqrt{T} \rceil}\right)  }+ \frac{\log(LKT)}{n_{i,t}^{\ell}} \right) \\
	&\stepb{\le} \gamma\left(\sqrt{6\log(LKT)\sum_{j\ge i}\frac{1}{n_{j,t}^{\ell}} \left(p_j + \frac{\gamma \log(LKT)}{\lceil \sqrt{T} \rceil}\right)  }+ \frac{\log(LKT)}{n_{i,t}^{\ell}} \right), 
\end{align*}
where (a) is due to \eqref{eq:p_0_lower}, and (b) is due to \eqref{eq:p_0_upper}. The proof of Lemma \ref{lemma.UCB} is then complete.

\subsection{Lemma \ref{lemma.combinatorial} and Its Proof}

The following Lemma \ref{lemma.combinatorial}, which holds deterministically under any level decomposition and any bidding sequence, is a key combinatorial lemma used in the proof of Lemma \ref{lemma:exploration}. 
\begin{lemma}\label{lemma.combinatorial}
	Let $n_{j,t}^{\ell}$ be defined in \eqref{eq:sample_size_level}, and $(v_1,\cdots,v_K)$ be any non-negative vector with $\sum_{i=1}^K v_i = s$. Then for any $\ell\in [L]$ and any subset $\Phi^{\ell}(T) \subseteq [T]$, it holds that
	\begin{align*}
	\sum_{t>T_0: t\in \Phi^{\ell}(T)}\sum_{j\ge i(t)} \frac{v_j}{n_{j,t}^{\ell} } \le s(1+\log T), \qquad \sum_{t>T_0: t\in \Phi^{\ell}(T)}\frac{1}{n_{i(t),t}^{\ell}} \le K(1+\log T).
	\end{align*}
\end{lemma}

Recall from \eqref{eq:sample_size_level} and the initialization before $t\le T_0$ that
\begin{align*}
n_{j,t}^{\ell} = \sum_{s\in \Phi^{\ell}(T), s<t} \1(b_s \le b^j) = \lceil\sqrt{T} \rceil +  \sum_{s\in \Phi^{\ell}(T), T_0<s<t} \1(b_s \le b^j). 
\end{align*}
Also, as $b_t = b^{i(t)}$, the condition $j\ge i(t)$ is equivalent to $b_t\le b^j$. Hence, 
\begin{align*}
\sum_{t>T_0: t\in \Phi^{\ell}(T)}\sum_{j\ge i(t)} \frac{v_j}{n_{j,t}^{\ell} } &= \sum_{j=1}^K v_j \sum_{t>T_0: t\in \Phi^{\ell}(T)} \frac{\1(j\ge i(t))}{n_{j,t}^{\ell}} \\
&= \sum_{j=1}^K v_j \sum_{t>T_0: t\in \Phi^{\ell}(T)} \frac{\1(b_t \le b^j)}{n_{j,t}^{\ell}} \\
&= \sum_{j=1}^K v_j\sum_{t>T_0: t\in \Phi^{\ell}(T)} \frac{\1(b_t \le b^j)}{\lceil\sqrt{T} \rceil + \sum_{s\in \Phi^{\ell}(T), T_0<s<t} \1(b_s \le b^j) } \\
&\le \sum_{j=1}^K v_j\left(\frac{1}{\lceil  \sqrt{T} \rceil} + \frac{1}{\lceil  \sqrt{T} \rceil + 1} + \cdots + \frac{1}{T}\right) \\
&\le \sum_{j=1}^K v_j(1+\log T) = s(1+\log T), 
\end{align*}
completing the proof of the first inequality. The second inequality follows from the first by choosing $(v_1,\cdots,v_K)=(1,\cdots,1)$.